\documentclass{article}
\usepackage{kr}

\usepackage{times}
\usepackage{soul}
\usepackage{url}
\usepackage[hidelinks]{hyperref}
\usepackage[utf8]{inputenc}
\usepackage[small]{caption}
\usepackage{graphicx}
\usepackage{amsmath}
\usepackage{amsthm}
\usepackage{booktabs}
\usepackage{algorithm}
\usepackage{algorithmic}
\usepackage{amssymb}
\usepackage{multirow}
\usepackage{natbib}
\usepackage{xcolor,caption,subcaption,graphicx}

\usepackage{bbm}
\usepackage[all]{xy}
\usepackage{tikz}
\usetikzlibrary{arrows.meta}
\usetikzlibrary{positioning,shapes,arrows}

\newcommand{\modelName}{\textsc{ReshufflE}}

\newtheorem{example}{Example}

\newtheorem{proposition}{Proposition}
\newtheorem{lemma}{Lemma}
\newtheorem{corollary}{Corollary}
\newtheorem{definition}{Definition}

\title{Faithful Differentiable Reasoning with Reshuffled Region-based Embeddings}

\author{%
Aleksandar Pavlovic$^1$\and Emanuel Sallinger$^{1,2}$\and Steven Schockaert$^3$\\ 
\affiliations
$^1$TU Wien, Vienna, Austria\\
$^2$University of Oxford, Oxford, UK\\
$^3$Cardiff University, Cardiff, UK\\
\emails
aleksandar.pavlovic.ai@gmail.com,
emanuel.sallinger@tuwien.ac.at,
schockaerts1@cardiff.ac.uk
}

\begin{document}

\maketitle

\begin{abstract}
Knowledge graph (KG) embedding methods learn geometric representations of entities and relations to predict plausible missing knowledge. These representations are typically assumed to capture rule-like inference patterns. However, our theoretical understanding of which inference patterns can be captured  remains limited. Ideally, KG embedding methods should be expressive enough such that for any set of rules, there exist relation embeddings that exactly capture these rules. This principle has been studied within the framework of region-based embeddings, but existing models are severely limited in the kinds of rule bases that can be captured. We argue that this stems from the fact that entity embeddings are only compared in a coordinate-wise fashion. As an alternative, we propose \modelName, a simple model based on ordering constraints that can faithfully capture a much larger class of rule bases than existing approaches. Most notably, {\modelName} can capture bounded inference w.r.t.\ arbitrary sets of closed path rules. The entity embeddings in our framework can be learned by a Graph Neural Network (GNN), which effectively acts as a differentiable rule base. 
\end{abstract}

\section{Introduction}
Knowledge graph (KG) embeddings \citep{DBLP:conf/nips/BordesUGWY13,DBLP:journals/corr/YangYHGD14a,DBLP:conf/icml/TrouillonWRGB16,DBLP:conf/iclr/SunDNT19} are geometric representations of knowledge graphs. Such representations are typically used to infer plausible knowledge that is not explicitly stated in the KG. An important research question is concerned with the kinds of regularities that can be captured by different kinds of models. While most standard approaches are difficult to analyse from this perspective, region-based approaches make these regularities more explicit \citep{DBLP:conf/kr/Gutierrez-Basulto18,DBLP:conf/nips/AbboudCLS20,DBLP:conf/iclr/0002S23,DBLP:conf/ijcai/CharpenayS24}. Essentially, in such approaches, each entity $e$ is represented by an embedding $\mathbf{e}\in\mathbb{R}^d$ and each relation $r$ is represented by a geometric region $Z_r\subseteq \mathbb{R}^{2d}$. The triple $(e,r,f)$ is then captured by the embedding iff $\mathbf{e}\oplus\mathbf{f}\in Z_r$, where $\oplus$ denotes vector concatenation. In this way, we can naturally associate a KG with a given embedding. Region-based models can also associate a \emph{rule base} with the embedding, where the rules are reflected in the spatial configuration of the regions $Z_r$. However, not all rule bases can be captured in this way. As a simple example, models based on TransE \citep{DBLP:conf/nips/BordesUGWY13} cannot distinguish between the rules $r_1(X_1,X_2)\wedge r_2(X_2,X_3)\rightarrow r_3(X_1,X_3)$ and $r_2(X_1,X_2)\wedge r_1(X_2,X_3)\rightarrow r_3(X_1,X_3)$. 
This particular limitation can be avoided by using more sophisticated region-based models \citep{DBLP:conf/iclr/0002S23,DBLP:conf/ijcai/CharpenayS24}, but even these models can only capture particular kinds of rule bases. This appears to be related to the fact that they rely on regions which can be characterised in terms of $d$ two-dimensional regions $Z^r_1,..., Z^r_d$, with $Z_i^r \subseteq \mathbb{R}^2$. To check whether $(e,r,f)$ is captured, we then check whether $(e_i,f_i)\in Z_i^r$ for each $i\in \{1,...,d\}$, with $\mathbf{e}= (e_1,...,e_d)$ and $\mathbf{f}= (f_1,...,f_d)$. We will refer to such approaches as \emph{coordinate-wise} models.
%\footnote{BoxE \citep{DBLP:conf/nips/AbboudCLS20} follows a slightly different approach, which we can still consider to be coordinate-wise, but where two different embeddings are used for each entity.} 
Existing models primarily differ in how these two-dimensional regions are defined, e.g.\ ExpressivE \citep{DBLP:conf/iclr/0002S23} uses parallelograms for this purpose, while \citet{DBLP:conf/ijcai/CharpenayS24} used octagons. 

In this paper, we propose a model that goes beyond coordinate-wise comparisons, which we term \modelName. A key challenge in designing such a model is that more flexible representations typically lead to overfitting. We avoid this problem by otherwise keeping the model as simple as possible, learning regions which are defined in terms of ordering constraints of the form $e_i\leq f_{j}$.
As our main contribution, we show that {\modelName} is more expressive than existing region-based models. 
%For instance, to the best of our knowledge, {\modelName} is the first (practical) region-based model that can capture (some) rule bases with cyclic dependencies. 
Furthermore, we show how entity embeddings can be learned using a Graph Neural Network (GNN) with randomly initialised node embeddings. This GNN effectively serves as a differentiable approximation of a rule base, acting on the initial representations of the entities to ensure that they capture the consequences that can be inferred from the KG. 
A practical consequence is that entity embeddings can thus be efficiently updated when new knowledge becomes available. From a theoretical point of view, the GNN-based formulation allows us to study bounded inference, where the number of layers of the GNN can be related to the number of inference steps. 
%In particular, we show that our model is capable of faithfully capturing bounded inference with arbitrary sets of closed path rules. Finally, while the main focus of this paper is on advancing our theoretical understanding of the expressivity of knowledge graph embeddings, we also empirically evaluate {\modelName} on the task of inductive KG completion, where we find that it outperforms existing differentiable rule learning strategies.

%**************************************************************************
\section{Related Work}\label{secRelatedWork}
\paragraph{Region-based Models} Our theoretical understanding of the reasoning abilities of KG embedding models remains poorly understood. This topic has primarily been studied in a line of work that uses region-based representations of relations \citep{DBLP:conf/kr/Gutierrez-Basulto18,DBLP:conf/nips/AbboudCLS20,DBLP:conf/nips/ZhangWCJW21,DBLP:journals/amai/LeemhuisOW22,DBLP:conf/iclr/0002S23,DBLP:conf/ijcai/CharpenayS24}. Essentially, the region-based view makes explicit which triples and rules are captured by a given embedding, which allows us to formally study what kinds of semantic dependencies a given model is capable of capturing \citep{DBLP:conf/kr/Gutierrez-Basulto18,DBLP:conf/nips/AbboudCLS20,DBLP:conf/kr/Bourgaux0KLO24}. Existing work has uncovered various limitations of popular KG embedding models. For instance, \citet{DBLP:conf/kr/Gutierrez-Basulto18} revealed that bilinear models such as RESCAL \citep{DBLP:conf/icml/NickelTK11}, DistMult \citep{DBLP:journals/corr/YangYHGD14a}, TuckER \citep{DBLP:conf/emnlp/BalazevicAH19} and ComplEx \citep{DBLP:conf/icml/TrouillonWRGB16} cannot capture relation hierarchies in a faithful way. They also studied the expressivity of models that represent relations using convex polytopes, finding that arbitrary sets of closed path rules can be faithfully captured by such representations (among others). However, learning arbitrary polytopes is not feasible for high-dimensional spaces, hence more recent works have focused on finding regions that are easier to learn while still retaining some of the theoretical advantages, such as Cartesian products of boxes \citep{DBLP:conf/nips/AbboudCLS20}, cones \citep{DBLP:conf/nips/ZhangWCJW21,DBLP:journals/amai/LeemhuisOW22},  parallelograms \citep{DBLP:conf/iclr/0002S23} and octagons \citep{DBLP:conf/ijcai/CharpenayS24}. 
However, all these models are significantly more limited in the kinds of rules that they can capture. For instance, while the use of parallelograms and octagons makes it possible to capture closed path rules, in practice we want to capture  \emph{sets} of such rules. This is only known to be possible under rather restrictive conditions (see Section \ref{secProblemSetting}).

An important practical advantage of region-based models is that they enable a tight integration of symbolic rules and vector space embeddings. This makes it possible to ``inject'' prior knowledge in a principled way \citep{DBLP:conf/nips/AbboudCLS20} and to inspect the kinds of rules that a given model has captured. A number of embedding based approaches have been proposed with similar advantages. For instance, some methods leverage symbolic rules to regularise the embedding space \citep{DBLP:conf/emnlp/GuoWWWG16,DBLP:conf/acl/TangZLZ24}. Neuro-symbolic methods which jointly learn a (differentiable approximation of) a Markov Logic Network with a KG embedding have also been proposed \citep{DBLP:conf/nips/Qu019,DBLP:conf/nips/ChenCFHS23}. However, note that these approaches are still limited by the expressivity of the underlying KG embedding model. For instance, DiffLogic \citep{DBLP:conf/nips/ChenCFHS23} aligns a Probablistic Soft Logic \citep{DBLP:journals/jmlr/BachBHG17} theory with a RotatE embedding \citep{DBLP:conf/iclr/SunDNT19}. RotatE, like TransE, cannot distinguish between the rules $r_1(X_1,X_2)\wedge r_2(X_2,X_3)\rightarrow r_3(X_1,X_3)$ and $r_2(X_1,X_2)\wedge r_1(X_2,X_3)\rightarrow r_3(X_1,X_3)$, hence this limitation is carried over to DiffLogic.

%Meng Qu, Jian Tang: Probabilistic Logic Neural Networks for Reasoning. NeurIPS 2019: 7710-7720

%Shengyuan Chen, Yunfeng Cai, Huang Fang, Xiao Huang, Mingming Sun: Differentiable Neuro-Symbolic Reasoning on Large-Scale Knowledge Graphs. NeurIPS 2023

%Xiaojuan Tang, Song-Chun Zhu, Yitao Liang, Muhan Zhang: RulE: Knowledge Graph Reasoning with Rule Embedding. ACL (Findings) 2024: 4316-4335

%***************************************
\paragraph{Inductive KG Completion}
Standard benchmarks for KG completion only test the reasoning abilities of models to a limited extent. 
%For instance, BoxE \citep{DBLP:conf/nips/AbboudCLS20} achieves strong results on these benchmarks, despite provably being incapable of modelling simple rules such as $r_1(X,Y)\wedge r_2(Y,Z)\rightarrow r_3(X,Z)$.
In our experiments, we will therefore focus on the problem of \emph{inductive} KG completion \citep{DBLP:conf/icml/TeruDH20}. In the inductive setting, we need to predict links between entities that are different from those that were seen during training. To perform this task, models need to learn semantic dependencies between the relations, and then exploit this knowledge when making predictions. 
%This can be achieved in different ways. 
A natural strategy is to learn rules from the training KG, either explicitly using models such as AnyBURL \citep{DBLP:conf/ijcai/MeilickeCRS19} and RNNLogic \citep{DBLP:conf/iclr/QuCXBT21} or implicitly using differentiable rule learners such as Neural-LP \citep{DBLP:conf/nips/YangYC17}, DRUM \citep{DBLP:conf/nips/SadeghianADW19} and NCRL \citep{DBLP:conf/iclr/ChengAS23}.  

Other approaches reduce link prediction to a graph classification problem \citep{DBLP:conf/icml/TeruDH20}.  However, this requires constructing and processing a different graph for each candidate tail entity, which is inherently inefficient. NBFNet \citep{DBLP:conf/nips/ZhuZXT21} alleviates this by processing the entire graph with a single forward pass of a GNN. RED-GNN \citep{DBLP:conf/www/ZhangY22} follows a similar approach, while A$^*$Net \citep{DBLP:conf/nips/ZhuYGX0G023} uses a learned heuristic to avoid processing the entire graph, providing further efficiency gains.
%The resulting node embeddings can then be used to score the different candidate tail entities. 
However, in all these models, the node embeddings are query-specific, meaning that a new forward pass of the GNN is still needed for each query, which is less efficient than using KG embeddings. MorsE \citep{DBLP:conf/sigir/Chen0ZZYXC22} addresses this limitation by using a GNN to compute embeddings for previously unseen entities, and then using the embeddings for link prediction (and other tasks). We adopt a similar approach in this paper. However, in our case, each layer of the GNN essentially simulates the application of a rule base, making our method conceptually closer to  differentiable rule learning methods. % than with subgraph classification strategies. 
ReFactor GNN \citep{DBLP:conf/nips/ChenM0MS022} also uses a GNN to learn entity embeddings, by simulating the training dynamic of traditional KG embedding methods such as TransE \citep{DBLP:conf/nips/BordesUGWY13},  although their method has the disadvantage that all embeddings have to be recomputed when new triples are added to the KG. Moreover, it inherits the limitations of traditional embedding models when it comes to faithfully modelling rules.

%Finally, while we focus on methods that learn by analysing the structure of a given KG, in some domains it is also possible to exploit prior knowledge, for instance by leveraging LLMs \citep{DBLP:journals/www/ZhuWCQOYDCZ24}.

%       Needs ``reasoning'': can be done using GNNs, rules or differentiable rule based approximations
%       GNNs for this problem are essentially doing differentiable reasoning, but main difference is that they focus on a single head entity, which is different from immediate conseuqence style reasoning
%       Discuss paper Akash as evidence for why rule based methods matter even when they don't reach state-of-the-art performance

% Existing methods mostly about one-off predictions, not attempting to do deductive chaining, whereas our GNN is capable of chaining rules (in theory)

%**************************************************************************
\section{Problem Setting}\label{secProblemSetting}
The focus of this paper is on studying which kinds of rule bases can be faithfully captured by the proposed model. In this section, we first formally define what it means for a region-based embedding model to capture a rule base. 

\paragraph{Preliminaries}
Let $\mathcal{R}$ be a set of relations, $\mathcal{E}$ a set of entities and $\mathcal{G}\subseteq \mathcal{E}\times\mathcal{R}\times\mathcal{E}$ a knowledge graph. If $\mathcal{G}$ contains the triple $(e,r,f)$ then we also say that there is an $r$-edge from $e$ to $f$ in $\mathcal{G}$. An entity embedding $\tau$ maps each entity $e\in\mathcal{E}$ to a vector $\tau(e)\in\mathbb{R}^d$. A region-based relation embedding $\eta$ maps each relation $r\in\mathcal{R}$ to a geometric region $\eta(r)\subseteq \mathbb{R}^{2d}$. 
\begin{definition}
We say that the triple $(e,r,f)$ is captured by the pair $(\tau,\eta)$, with $\tau$ an entity embedding and $\eta$ a region-based relation embedding, iff $\tau(e)\oplus\tau(f)\in \eta(r)$. 
\end{definition}
\noindent We write $\mathcal{P}\cup \mathcal{G} \models (e,r,f)$ to denote that the triple $(e,r,f)$ can be entailed from the rule base $\mathcal{P}$ and the knowledge graph $\mathcal{G}$. More precisely, we have $\mathcal{P}\cup \mathcal{G} \models (e,r,f)$ iff either $(e,r,f)\in \mathcal{G}$, or $\mathcal{P}$ contains a rule $r_1(X_1,X_2) \wedge r_2(X_2,X_3) \wedge ... \wedge r_p(X_p,X_{p+1}) \rightarrow r(X_1,X_{p+1})$ such that $\mathcal{P}\cup \mathcal{G} \models (e,r_1,e_2)$, $\mathcal{P}\cup \mathcal{G} \models (e_2,r_2,e_3)$, ..., $\mathcal{P}\cup \mathcal{G} \models (e_p,r_p,f)$ for some entities $e_2,...,e_p$. We furthermore write $\mathcal{P} \models \rho$, for a rule $\rho$, to denote that $\mathcal{P}$ entails $\rho$ w.r.t.\ the standard notion of entailment from first-order logic (when viewing rules as universally quantified material implications). 

 %In other words, in practice $\mathcal{G}$ and $\mathcal{P}$ is (implicitly) learned.
% For instance, a region-based variant of TransE can be obtained by learning regions of the following form:
% \begin{align}\label{eqTransERegions}
% \eta(r) = \{\mathbf{e}\oplus\mathbf{f}\,|\, \forall i\in\{1,...,d\}\,.\, l_r^i\leq f_i-e_i \leq u_r^i\}
% \end{align}
% where we assume $\mathbf{e}=(e_1,...,e_d)$ and $\mathbf{f}=(f_1,...,f_d)$, and $l_r^1,...,l_r^d,u_r^1,...,u_r^d$ are learnable parameters. Note how the region $Z_r$ is defined coordinate-wise, using 2-dimensional regions with a simple shape.
%, where we write $\mathbf{e}\oplus\mathbf{f}$ to denote vector concatenation. 
%The scoring function $s_r$ then reflects how close $\mathbf{e}\oplus\mathbf{f}$ is to the region $X_r$ (which is formalised in different ways by different models). 
%
%
%A key advantage of region-based relation embeddings is that they offer a mechanism for modelling rules. Let us write $\eta$ to denote a given region-based embedding, i.e.\ %$\eta(e)\in\mathbb{R}^d$ (for $e\in\mathcal{E}$) denotes the embedding of the entity $e$ and 
%$\eta(r)\subseteq \mathbb{R}^{2d}$ denotes the region representing $r\in\mathcal{R}$. 
%We say that $\eta$ captures a triple $(e,r,f)$, written $\eta\models (e,r,f)$ iff $\eta(e)\oplus\eta(f)\in \eta(r)$.  Now 

\paragraph{Capturing Closed Path Rules}
Similar to most existing rule-based methods for KG completion \citep{DBLP:conf/nips/YangYC17,DBLP:conf/ijcai/MeilickeCRS19,DBLP:conf/nips/SadeghianADW19,DBLP:conf/iclr/QuCXBT21,DBLP:conf/iclr/ChengAS23}, we focus on \emph{closed path rules}, which are rules $\rho$ of the form:
\begin{align}
r_1(X_1,X_2) &\wedge r_2(X_2,X_3) \wedge \notag\\
& ...\wedge r_p(X_p,X_{p+1}) \rightarrow r(X_1,X_{p+1})\label{eqClosedPathRule}
\end{align}
We refer to $r_1(X_1,X_2) \wedge r_2(X_2,X_3) \wedge ... \wedge r_p(X_p,X_{p+1})$ as the body of the rule and to $r(X_1,X_{p+1})$ as the head.
\begin{definition}\label{eqRuleCaptured}
We say that a region-based relation embedding $\eta$ captures a rule of the form \eqref{eqClosedPathRule} if for all vectors $\mathbf{x_1},...,\mathbf{x_{p+1}}\in\mathbb{R}^d$ such that that $(\mathbf{x_1}\oplus \mathbf{x_2}\in \eta(r_1)) \wedge ....\wedge (\mathbf{x_p}\oplus \mathbf{x_{p+1}}\in \eta(r_p))$ we also have   $(\mathbf{x_1}\oplus \mathbf{x_{p+1}}\in \eta(r))$.
\end{definition}
\noindent Apart from their practical significance, our focus on closed path rules is also motivated by the observation that existing region-based models have particular limitations when it comes to capturing this kind of rules. Some approaches, such as BoxE \citep{DBLP:conf/nips/AbboudCLS20} are not capable of capturing such rules at all. 
More recent approaches \citep{DBLP:conf/iclr/0002S23,DBLP:conf/ijcai/CharpenayS24} are capable of capturing closed path rules, but with significant limitations when it comes to modelling sets of such rules.

\paragraph{Eq-complete Knowledge Graphs}
Throughout this paper, we will assume that all knowledge graphs $\mathcal{G}$ contain the triple $(e,\textit{eq},e)$ for every entity $e$. We will refer to such knowledge graphs as \textit{eq}-complete. This assumption is similar to the common practice of adding self-loops in GNN models. In our setting, it will mean that instead of directly capturing a rule like $r_1(X_1,X_2)\wedge r_2(X_2,X_3)\rightarrow r(X_1,X_3)$ we might instead capture a rule like $r_1(X_1,X_2)\wedge \textit{eq}(X_2,X_3) \wedge r_2(X_3,X_4)\rightarrow r(X_1,X_4)$. This latter rule is equivalent, in the sense that any triple that can be inferred from an \textit{eq}-complete KG with the former rule, can also be inferred with the latter rule. This offers some modelling flexibility which will be important in our approach. 

\paragraph{Faithfully Capturing Rule Bases}
Definition \ref{eqRuleCaptured} specifies what it means for a closed path rule to be captured.
Our main research question is whether it is possible to faithfully capture a \emph{set} of closed path rules $\mathcal{P}$. In other words, can parameters be found for the matrices $\mathbf{B_r}$ such that \emph{all rules} entailed by $\mathcal{P}$ are captured, and \emph{only those rules}. This is made precise in the following definition.\footnote{Our notion of faithfully capturing rule bases is closely related to the notion of exactly and exclusively capturing a language of patterns, from \cite{DBLP:conf/nips/AbboudCLS20}, and the notion of strong TBox-faithfulness from \cite{DBLP:conf/kr/Bourgaux0KLO24}. It is, however, slightly weaker due to restriction to \textit{eq}-complete knowledge graphs.}
\begin{definition}\label{defFaithfulRulBase}
We say that a region-based relation embedding $\eta$ faithfully captures the rule base $\mathcal{P}$ if for every \textit{eq}-complete knowledge graph $\mathcal{G}$, the following conditions hold:
\begin{enumerate}
\item Suppose that $\mathcal{P}\cup\mathcal{G}\models (a,r,b)$ and 
let $\tau$ be an entity embedding such that $(\tau,\eta)$ captures every triple in $\mathcal{G}$. Then $(\tau,\eta)$ captures the triple $(a,r,b)$ as well.
\item Suppose that $\mathcal{P}\cup\mathcal{G}\not\models (a,r,b)$. There exists an entity embedding $\tau$ such that $(\tau,\eta)$ captures every triple in $\mathcal{G}$ but not the triple $(a,r,b)$.
\end{enumerate}
\end{definition}
\noindent Existing models are only able to faithfully capture sets of closed path rules in specific cases. For instance, \citet{DBLP:conf/ijcai/CharpenayS24} showed this to be possible when every non-trivial rule entailed from $\mathcal{P}$ is a closed path rule of the form \eqref{eqClosedPathRule} in which $r_1,...,r_p,r$ are all distinct relations. 
For instance, rules of the form $r_1(X_1,X_2)\wedge r_1(X_2,X_3)\rightarrow r(X_1,X_3)$ 
%and $r_1(X_1,X_2)\wedge r_2(X_2,X_3)\rightarrow r_1(X_1,X_3)$ 
are not covered by their result. 
Similarly, they cannot capture rule bases with cyclic dependencies such as: 
\begin{align*}
\mathcal{P} = \{&r_1(X_1,X_2)\wedge r_2(X_2,X_3)\rightarrow r_3(X_1,X_3),\\
&r_3(X_1,X_2)\wedge r_4(X_2,X_3)\rightarrow r_1(X_1,X_3)\}
\end{align*}
Note that while we focus our theoretical analysis on whether a given rule base $\mathcal{P}$ can be captured, in practice we normally do not have access to such a rule base. We study whether our model is capable of capturing rule bases because this is a necessary condition to allow it to \emph{learn} semantic dependencies in the form of  rules.
%**************************************************************************
% \section{Modeling relations using ordering constraints}\label{secModelDescription}
\section{Model Formulation}\label{secModelDescription}
Our aim is to introduce a knowledge graph embedding model which goes beyond coordinate-wise models, but which otherwise remains as simple as possible. The central idea is to define the regions $\eta(r)$ using coordinate-wise ordering constraints between reshuffled entity embeddings: %. Specifically, we model each relation $r$ using a region $X_r$ of the following form:
\begin{align}\label{eqOrderingConstraint}
\eta(r) = \{ (e_1,...,e_d,f_1,...,f_d) \,|\, \forall i\in I_r\,.\, e_{\sigma_r(i)} \leq f_i\}
\end{align}
where the representation of a region $r$ is parameterised by a set of coordinates $I_r\subseteq \{1,...,d\}$ and a mapping $\sigma_r:I_r\rightarrow \{1,...,d\}$. We thus need a maximum of $2d$ parameters to completely specify the embedding of a given relation.
Note that in the special case where $I_r=\emptyset$, we have $\eta(r)= \mathbb{R}^{2d}$. 
\begin{example}\label{exInitialExampleReshuffleConstraint}
Let $\mathbf{e}=(0,0,0)$, $\mathbf{f}=(0,0,1)$ and $\mathbf{g}=(2,2,0)$ be the embeddings of entities $e,f,g$.
Let the relations $r_1, r_2, r_3$ be represented as follows: $I_{r_1} =\{3\}$, $I_{r_2} =\{1,2\}$, $I_{r_3}=\{1\}$, $\sigma_{r_1}(3)=2$, $\sigma_{r_2}(1)=\sigma_{r_2}(2)=3$ and $\sigma_{r_3}(1)=2$. Then we find that $\mathbf{e}\oplus\mathbf{f} \in \eta(r_1)$, meaning that the triple $(e,r_1,f)$ is captured. Indeed, for $\mathbf{e}\oplus\mathbf{f} \in \eta(r_1)$ to hold, we need $e_{\sigma_{r_1}(3)}= e_2 \leq f_3$, which is satisfied. We similarly find that $(f,r_2,g)$ is captured, because $f_{\sigma_{r_2}(1)}= f_3 \leq g_1$ and $f_{\sigma_{r_2}(2)}= f_3 \leq g_2$.
\end{example}
\noindent The following example illustrates how the use of ordering constraints allows us to capture rules.
\begin{example}\label{exIntuitionsOrderingConstraints}
Consider a rule $r_1(X,Y)\wedge r_2(Y,Z) \rightarrow r_3(X,Z)$. This rule is captured by an embedding of the form \eqref{eqOrderingConstraint} if for each $i\in I_{r_3}$ we have that $i\in I_{r_2}$, $\sigma_{r_2}(i)\in I_{r_1}$ and $\sigma_{r_1}(\sigma_{r_2}(i)) = \sigma_{r_3}(i)$. For instance, the relations $r_1, r_2, r_3$ from Example \ref{exInitialExampleReshuffleConstraint} satisfy these conditions. In general, if these conditions are satisfied and the triples $(e,r_1,f)$ and $(f,r_2,g)$ are captured, then for each $i\in I_{r_3}$ we have:
$e_{\sigma_{r_1}(\sigma_{r_2}(i))} \leq f_{\sigma_{r_2}(i)} \leq g_i$.
Since we assumed $\sigma_{r_1}(\sigma_{r_2}(i))=\sigma_{r_3}(i)$ it follows that $e_{\sigma_{r_3}(i)}\leq g_i$ for every $i\in I_r$ and thus that the triple $(e,r_3,f)$ is also captured. 
\end{example}

\paragraph{Characterising Ordering Constraints}
We now turn our focus to how the proposed model can be represented in a way that enables efficient GPU computations. Note that we can characterise $\eta(r)$ as follows: 
\begin{align}\label{eqCharacterisationReshuffleMatrices}
\eta(r) = \{\mathbf{e}\oplus\mathbf{f} \,|\, \max(\mathbf{A_r}\mathbf{e},\mathbf{f})=\mathbf{f}\}
\end{align}
where the maximum is applied component-wise and the matrix $\mathbf{A_r}\in \mathbb{R}^{d\times d}$ is constrained such that (i) all components are either 0 or 1 and (ii) at most one component in each row is non-zero. 
As we explain below, the entity embeddings will be constructed using a GNN, starting from randomly initialised embeddings. For this approach to be effective, the dimensionality of the entity embeddings needs to be sufficiently high, to prevent randomly initialised embeddings from capturing KG triples by chance. At the same time, the number of parameters should remain sufficiently low to prevent overfitting. For this reason, we decouple the number of parameters from the dimensionality of the embeddings. To this end, we learn matrices $\mathbf{A_r}$ of the following form:
\begin{align}\label{eqLearningAKronecker}
\mathbf{A_r} = \mathbf{B_r} \otimes \mathbf{I_k}
\end{align}
where we write $\otimes$ for the Kronecker product, $\mathbf{I_k}$ is the $k$-dimensional identity matrix and $\mathbf{B_r}$ is an $\ell\times\ell$ matrix, with $d=k \ell$. The rows of $\mathbf{B_r}$ are constrained similarly as those of $\mathbf{A_r}$, i.e.\ each row is either a one-hot vector or a 0-vector. To make the computations more efficient, we represent each entity using a matrix $\mathbf{Z_e}\in \mathbb{R}^{\ell\times k}$, rather than a vector, and define the region-based embeddings as follows:
\begin{align}\label{eqCharacterisationReshuffleFinal}
\eta(r) = \{\textit{flatten}(\mathbf{Z_e})\oplus\textit{flatten}(\mathbf{Z_f}) \,|\, \max(\mathbf{B_r}\mathbf{Z_e},\mathbf{Z_f})=\mathbf{Z_f}\}
\end{align}
where we write $\textit{flatten}(\mathbf{Z_e})$ for the vector that is obtained by concatenating the rows of $\mathbf{Z_e}$.
We will refer to the region-based embedding model defined in \eqref{eqCharacterisationReshuffleFinal} as \modelName.
Note that a triple $(e,r,f)$ is captured if $\mathbf{B_r}\mathbf{Z_e}\preceq \mathbf{Z_f}$,
where $\mathbf{X} \preceq \mathbf{Y}$ iff 
%no element of the matrix $\mathbf{X}$ is strictly larger than the corresponding element in $ \mathbf{Y}$. %, i.e.\ to denote that 
$\max(\mathbf{X},\mathbf{Y})=\mathbf{Y}$. It is furthermore easy to verify that a rule of the form \eqref{eqClosedPathRule} is satisfied if $\mathbf{B_{r}} \preceq \mathbf{B_{r_p}}\mathbf{B_{r_{p-1}}}\cdots \mathbf{B_{r_{1}}}$. 
%In the following section, we will introduce a graph representation that will allow us to study the rules captured by a given embedding more easily.
%

%***********************************
\paragraph{Learning Entity Embeddings with GNNs}
The format of \eqref{eqCharacterisationReshuffleFinal} suggests how entity embeddings in our framework can be learned using a GNN. A practical advantage of using a GNN for this purpose is that we can use our model for inductive KG completion. As we will see in Section \ref{secBounded}, the use of a GNN also has an important theoretical advantage, as it allows us to capture bounded reasoning.

Let us write $\mathbf{Z_e^{(l)}} \in\mathbb{R}^{\ell\times k}$ for the representation of entity $e$ in layer $l$ of the GNN. Starting from \eqref{eqCharacterisationReshuffleFinal}, we naturally end up with the following message-passing GNN:
\begin{align} \label{eqGNNformB}
\mathbf{Z_f^{(l+1)}} = \max\big(\{\mathbf{Z_f^{(l)}}\} \cup \{\mathbf{B_r}\mathbf{Z_e^{(l)}}\,|\, (e,r,f)\in \mathcal{G}\}\big)
\end{align}
The embeddings $\mathbf{Z_e^{(0)}}$ are initialised randomly, such that (i) all coordinates are non-negative, (ii) the coordinates of different entity embeddings are sampled independently, and (iii) there are at least two distinct values that have a non-negative probability of being sampled for each coordinate. In this way, the probability of a constraint $\mathbf{B_r}\mathbf{Z_e^{(0)}}\preceq \mathbf{Z_f^{(0)}}$ being satisfied by chance can be made arbitrarily small, without increasing the number of parameters of the model, by choosing the value of $k$ to be sufficiently high.

%where the matrices $\mathbf{A_r}$ are constrained as before. Due to this constraint, the GNN converges after a finite number of steps $m$ to embeddings $\mathbf{f^{m}}$ satisfying $\mathbf{f^{m}}=\max(\mathbf{f^{m}},\mathbf{A_r}\mathbf{e^{m}})$ for each $(e,r,f)\in \mathcal{G}$.

%It is easy to verify that this model is equivalent to \eqref{eqGNNformA}. 
%Specifically, the matrix $\mathbf{Z_e^{(l)}}=(z_{ij})$ corresponding to the entity embedding $\mathbf{e^{(l)}}=(e_1,...,e_d)$ is defined as $z_{ij}= e_{(i-1)k+j}$, with $i\in\{1,...,\ell\}$ and $j\in \{1,...,k\}$. 
%In other words, the first row of $\mathbf{Z_e^{(l)}}$ contains the first $k$ coordinates of  $\mathbf{e^{(l)}}$, the second row contains the next $k$ coordinates, etc. 

%**************************************************************************
\section{Constructing Models from Rule Graphs}\label{secRuleGraphs}
For any knowledge graph $\mathcal{G}$, the GNN defined in \eqref{eqGNNformB} can be used to construct an entity embedding $\tau$ such that $(\tau,\eta)$ captures every triple from $\mathcal{G}$. To see this, first note that because of the use of the maximum, the GNN always converges after a finite number of iterations. Upon convergence, it is clear that every triple of $\mathcal{G}$ must indeed be satisfied. However, the resulting entity embeddings may also capture triples which are not in $\mathcal{G}$. Our central research question is whether we can always choose the matrices $\mathbf{B_r}$ such that a triple is captured by these entity embeddings iff the triple can be inferred from $\mathcal{G}\cup\mathcal{P}$, for a given rule base $\mathcal{P}$. In other words, given a set of closed path rules $\mathcal{P}$, can we always construct a {\modelName} model that faithfully captures it?
Rather than constructing the matrices $\mathbf{B_r}$ directly, we first introduce the notion of a rule graph, which will serve as a convenient abstraction. 

%******************************
\paragraph{Rule Graphs}
Let $\mathcal{P}$ be a set of closed path rules.
We associate with $\mathcal{P}$ a labelled multi-graph $\mathcal{H}$, i.e.\ a set of triples $(n_1,r,n_2)$. Note that this graph is formally equivalent to a knowledge graph, but the nodes in this case do not correspond to entities. Rather, as we will see, they correspond to the different rows/columns of the matrices $\mathbf{B_r}$.  A path in  $\mathcal{H}$ from $n_1$ to $n_{p+1}$ is a sequence of triples of the form $(n_1,r_1,n_2), (n_2,r_2,n_3), ..., (n_p,r_p,n_{p+1})$. The \emph{type} of this path is given by the sequence of relations $r_1;r_2;...;r_p$. The \textit{eq}-reduced type of the path is obtained by removing all occurrences \textit{eq} in $r_1;r_2;...;r_p$.  For instance, for a path of type $r_1;\textit{eq};\textit{eq};r_2;\textit{eq}$, the \textit{eq}-reduced type is $r_1;r_2$.

\begin{definition}
Let $\mathcal{R}$ be a set of relations, and let $\mathcal{P}$ be a set of rules defined over these relations.
A rule graph $\mathcal{H}$ for $\mathcal{P}$ and $\mathcal{R}$ is a labelled multi-graph, where the labels are taken from $\mathcal{R}$, such that the following properties are satisfied: 
\begin{description}
\item[\textbf{(R1)}] For every relation $r\in \mathcal{R}$, there is some edge in $\mathcal{H}$ labelled with $r$.
\item[\textbf{(R2)}] For every node $n$ in $\mathcal{H}$ and every $r\in\mathcal{R}$, it holds that $n$ has at most one incoming $r$-edge.
\item[\textbf{(R3)}] Suppose there is an $r$-edge in $\mathcal{H}$ from node $n_1$ to node $n_2$. Suppose furthermore that $\mathcal{P}\models r_1(X_1,X_2) \wedge r_2(X_2,X_3) \wedge ... \wedge r_p(X_p,X_{p+1}) \rightarrow r(X_1,X_{p+1})$. Then there is a path in $\mathcal{H}$ from $n_1$ to $n_2$ whose $\textit{eq}$-reduced type is $r_1;...;r_p$.
\item[\textbf{(R4)}] Suppose that for every $r$-edge, there is a path connecting the same nodes whose \textit{eq}-reduced type belongs to $\{(r_{11};...;r_{1p_1}),\allowbreak ...,\allowbreak (r_{q1};...;r_{qp_q})\}$. Then there is some $i\in\{1,...,q\}$ such that that $\mathcal{P}\models r_{i1}(X_1,X_2)\wedge ... \wedge r_{ip_i}(X_{p_i},X_{p_{i+1}})\rightarrow r(X_1,X_{p_{i+1}})$.
\end{description}
\end{definition}

\noindent When $\mathcal{R}$ is clear from the context, we will simply refer to $\mathcal{H}$ as a rule graph for $\mathcal{P}$.
The definition reflects the fact that a rule is captured when the ordering constraints associated with its body entail the ordering constraints associated with its head, as was illustrated in Example \ref{exIntuitionsOrderingConstraints}. Specifically, this is encoded by condition (R3). Condition (R4) is needed to ensure that \emph{only} the rules in $\mathcal{P}$ are captured. Note that we require (R4) to be satisfied for \emph{every} set of path types $\{(r_{11};...;r_{1p_1}), ..., (r_{q1};...;r_{qp_q})\}$ that satisfy the condition. The following example illustrates why (R4) is needed.
\begin{example}
Suppose there are two $r$-edges in the rule graph, namely between $n_1$ and $n_2$ and between $n_3$ and $n_4$. Suppose furthermore that there is an $r_1;r_2$ path connecting $n_1$ and $n_2$, and an $r_3;r_4$ path connecting $n_3$ and $n_4$. Suppose that the knowledge graph $\mathcal{G}$ contains the triples $(a,r_1,b),(b,r_2,c),(a,r_3,d),(d,r_4,c)$. Then, as we will see, the corresponding model would capture the triple $(a,r,c)$. In other words, the model would capture the rule $r_1(X_1,X_2)\wedge r_2(X_2,X_3)\wedge r_3(X_1,X_4)\wedge r_4(X_4,X_3)\rightarrow r(X_1,X_3)$, but such a rule can never be entailed from $\mathcal{P}$, since the latter only contains closed path rules. We can thus only allow such rule graphs if $r_1(X_1,X_2)\wedge r_2(X_2,X_3)\rightarrow r(X_1,X_3)$ or $r_3(X_1,X_4)\wedge r_4(X_4,X_3)\rightarrow r(X_1,X_3)$ is entailed from $\mathcal{P}$, which is what (R4) ensures.
\end{example}
\noindent Conditions (R1) and (R2) are needed because, in the construction we consider below, the nodes of the rule graph will correspond to the rows/columns of the matrices $\mathbf{B_r}$. Condition (R1) ensures that $\mathbf{B_r}$ contains at least one non-zero component for each relation $r$, while (R2) ensures that each row of $\mathbf{B_r}$ has at most one non-zero component. 

\begin{example}\label{exRuleGraph1}
Let $\mathcal{P}_1$ consist of the following rules:
\begin{align*}
r_1(X,Y) \wedge r_2(Y,Z) &\rightarrow r_3(X,Z)\\
r_4(X,Y) \wedge r_5(Y,Z) &\rightarrow r_2(X,Z)
\end{align*}
A corresponding rule graph is shown in Figure \ref{figRuleGraph1}. As an example with cyclic dependencies, let $\mathcal{P}_2$ consist of:
\begin{align*}
r_2(X,Y) \wedge r_3(Y,Z) &\rightarrow r_1(Y,Z)\\
r_1(X,Y) \wedge r_4(Y,Z) &\rightarrow r_2(X,Z)
\end{align*}
A corresponding rule graph is shown in Figure \ref{figRuleGraph2Bis}.

\begin{figure}[t]
\centering
\begin{tikzpicture}[scale=0.8]
\begin{scope}[every node/.style={thick,draw}]
    \node (A) at (0,0) {$n_1$};
    \node (B) at (0,2) {$n_2$};
    \node (C) at (2,0) {$n_3$};
    \node (D) at (2,2) {$n_4$};   
    \node (E) at (-2,0) {$n_5$};   
 
\end{scope}

\begin{scope}[>={Stealth[black]},
              %every node/.style={fill=white,circle},
              every edge/.style={draw=black,thick}]   
               \path [->] (A) edge node[left]  {$r_3$} (B);               
               \path [->] (A) edge node[above]  {$r_1$} (C);
               \path [->] (C) edge node[left]  {$r_2$} (B);
               \path [->] (C) edge node[right]  {$r_4$} (D);
               \path [->] (D) edge node[above]  {$r_5$} (B);   
               \path [->] (A) edge node[above]  {\textit{eq}} (E);   

\end{scope}
\end{tikzpicture}
\caption{Rule graph for $\mathcal{P}_1$. \label{figRuleGraph1}}
\end{figure}

\begin{figure}[t]
\centering
\begin{tikzpicture}[scale=0.8]
\begin{scope}[every node/.style={thick,draw}]
    \node (A) at (0,0) {$n_1$};
    \node (B) at (4,0) {$n_2$};
    \node (C) at (2,2) {$n_3$};
    \node (D) at (-2,0) {$n_4$};
\end{scope}

\begin{scope}[>={Stealth[black]},
              %every node/.style={fill=white,circle},
              every edge/.style={draw=black,thick}]   
            \path [->] (A) edge node[below]  {$r_1$} (B);             
            \path [->] (A) edge node[left]  {$r_2$} (C);            
            \path [->] (C) edge[bend left=25] node[left]    {$r_3$} (B);               
            \path [->] (B) edge[bend left=25] node[left]    {$r_4$} (C); 
            \path [->] (A) edge node[above]    {\textit{eq}} (D);                            
\end{scope}
\end{tikzpicture}
\caption{Rule graph for $\mathcal{P}_2$. \label{figRuleGraph2Bis}}
\end{figure}
\end{example}

%****************************************
\paragraph{Constructing Models}
Given a rule graph $\mathcal{H}$, we define the corresponding parameters of a {\modelName} model (i.e.\ the matrices $\mathbf{B_r}$) as follows. Each node from the rule graph is associated with one row/column of $\mathbf{B_r}$. Let $n_1,...,n_\ell$ be an enumeration of the nodes in the rule graph. The corresponding matrix $\mathbf{B_r}=(b_{ij})$ is defined as:
\begin{align}\label{eqDefGNNfromRuleGraph}
b_{ij} = 
\begin{cases}
1 & \text{if $\mathcal{H}$ has an $r$-edge from $n_j$ to $n_i$}\\
0 & \text{otherwise}
\end{cases}
\end{align}
Note that because of condition (R2), there will be at most one non-zero element in each row of $\mathbf{B_r}$, in accordance with the assumptions that we made in Section \ref{secModelDescription}. 

If a set of closed path rules $\mathcal{P}$ has a rule graph $\mathcal{H}$ then the corresponding {\modelName} model, defined as in \eqref{eqDefGNNfromRuleGraph}, faithfully captures $\mathcal{P}$. To prove this, we need to show that the two conditions from Definition \ref{defFaithfulRulBase} are satisfied. The following proposition shows that the first condition is satisfied.

% \begin{proposition}\label{propRuleGraphCompleteness}
% Let $\mathcal{P}$ be a rule base and $\mathcal{G}$ a knowledge graph.
% Suppose $\mathcal{P}\cup \mathcal{G} \models (a,r,b)$. Let $\mathcal{H}$ be a rule graph for $\mathcal{P}$ and let $\mathbf{Z_e^{(l)}}$ be the entity representations that are learned by the corresponding \normalfont{\modelName} model, as defined in \eqref{eqDefGNNfromRuleGraph}. Assume $\mathbf{Z_e^{(m)}}=\mathbf{Z_e^{(m+1)}}$ for every entity $e$ ($m\in\mathbb{N}$). It holds that $\mathbf{B_r}\mathbf{Z_a^{(m)}}\preceq \mathbf{Z_b^{(m)}}$.
% \end{proposition}
\begin{proposition}\label{propRuleGraphCompleteness}
Let $\mathcal{P}$ be a set of closed path rules and $\mathcal{G}$ an \textit{eq}-complete knowledge graph.
Suppose $\mathcal{P}\cup \mathcal{G} \models (a,r,b)$. Let $\mathcal{H}$ be a rule graph for $\mathcal{P}$ and let $\eta$ be the corresponding {\modelName} model. Let $\tau$ be an entity embedding such that $(\tau,\eta)$ captures every triple in $\mathcal{G}$. It holds that $(\tau,\eta)$ captures the triple $(a,r,b)$.
\end{proposition}
\noindent To show that the second condition of Definition \ref{defFaithfulRulBase} is also satisfied, we need to show the existence of a particular entity embedding. We show that the embedding constructed by the GNN in \eqref{eqGNNformB} satisfies this condition with a probability that can be made arbitrarily high. The reason why this embedding is not guaranteed to satisfy the condition is because, for any triple $(e,r,f)$, there is always a chance that it is captured by the model, even if $\mathcal{P}\cup\mathcal{G}\not\models (e,r,f)$, due to the fact that the entity embeddings are initialised randomly. However, by choosing the dimensionality of the entity embeddings to be sufficiently large, we can make the probability of this happening arbitrarily small. As before, we write $\ell$ to denote the number of rows in $\mathbf{Z}_e$ and $k$ for the number of columns. Note that the value of $k$ does not affect the number of parameters, since the size of the matrices $\mathbf{B_r}$ only depends on $\ell$ and the entity embeddings are randomly initialised. In practice, we can thus simply choose $k$ to be sufficiently large.

\begin{proposition}\label{propProbablyNoFalsePositives}
Let $\mathcal{P}$ be a set of closed path rules and $\mathcal{G}$ an \textit{eq}-complete knowledge graph.
Let $\mathcal{H}$ be a rule graph for $\mathcal{P}$ and let $\mathbf{Z_e^{(l)}}$ be the entity representations learned using the GNN \eqref{eqGNNformB} for the corresponding {\modelName} model. For any $\varepsilon>0$, there exists some $k_0\in \mathbb{N}$ such that, when $k\geq k_0$, for any  $m\in\mathbb{N}$  and $(a,r,b)\in \mathcal{E}\times\mathcal{R}\times\mathcal{E}$ such that $\mathcal{P}\cup \mathcal{G} \not\models (a,r,b)$, we have
\begin{align*}
\textit{Pr}[ \mathbf{B_r}\mathbf{Z_a^{(m)}}\preceq \mathbf{Z_b^{(m)}}] \leq \varepsilon
\end{align*}
\end{proposition}
\noindent Finally, there always exists an initialisation of the entity embeddings for which the embeddings learned by the GNN satisfy the second condition of Definition \ref{defFaithfulRulBase}, even when $k=1$. In particular, we have the following result
\begin{proposition}\label{propFaithfullyCapturedIfRuleGraph}
Let $\mathcal{P}$ be a set of closed path rules. Let $\mathcal{H}$ be a rule graph for $\mathcal{P}$ and let $\eta$ be the corresponding {\modelName} model.  It holds that $\eta$ faithfully captures $\mathcal{P}$.
\end{proposition}
%**************************************************************************
\section{Constructing Rule Graphs}\label{secConstructingRuleGraphs}
We now return to the central question of this paper: given a set of closed path rules $\mathcal{P}$, is it possible to construct a {\modelName} model which faithfully captures $\mathcal{P}$? Thanks to Proposition \ref{propFaithfullyCapturedIfRuleGraph}  we know that this is the case when a rule graph for $\mathcal{P}$ exists. The key question thus becomes whether it is always possible to construct such a rule graph. As the following result shows, if there are no cyclic dependencies in $\mathcal{P}$, a rule graph always exists.
\begin{proposition}\label{propAcyclicRuleBase}
Let $\mathcal{P}$ be a rule base. Assume that we can rank the relations in $\mathcal{R}$ as $r_1,...,r_{|\mathcal{R}|}$, such that for every rule in $\mathcal{P}$ with $r_i$ in the body and $r_j$ in the head, it holds that $i<j$. There exists a rule graph for $\mathcal{P}$.
\end{proposition}
\noindent It follows that the class of rule bases that can be captured by {\modelName} models is strictly larger than the class that has been considered in previous work \citep{DBLP:conf/ijcai/CharpenayS24}.
Unfortunately,  there exist rule bases with cyclic dependencies for which no valid rule graph can be found. %This is illustrated in the next example.
\begin{example}\label{exRuleGraph3}
Let $\mathcal{P}$ contain the following rule:
\begin{align*}
r_1(X,Y) \wedge r_2(Y,Z) \wedge r_1(Z,U) &\rightarrow r_2(X,U)
\end{align*}
To see why there is no rule graph for $\mathcal{P}$, consider the following knowledge graph $\mathcal{G}$:
\begin{align*}
\mathcal{G}{=}\{&(x_1,r_1,x_2), (x_2,r_1,x_3), ...,(x_{l-1},r_1,r_{l}),\\
&(x_l,r_2,x_{l+1}),(x_{l+1},r_1,x_{l+2}),...,(x_{k},r_1,x_{k+1})\}
\end{align*}
We have that $\mathcal{P}\cup\mathcal{G}\models (x_1,r_2,x_{k+1})$ only if the number of repetitions of $r_1$ at the start of the sequence matches the number at the end, but rule graphs cannot encode this.
\end{example}

\noindent The argument from the previous example can be formalised as follows. Let $\mathcal{P}$ be a set of closed path rules. Let $\mathcal{R}_1$ be the set of relations from $\mathcal{R}$ that appear in the head of some rule in $\mathcal{P}$. For any $r\in \mathcal{R}_1$, we can consider a context-free grammar with two types of production rules:
\begin{itemize}
\item For each rule $r_1(X_1,X_2) \wedge  ... \wedge r_p(X_p,X_{p+1}) \rightarrow r(X_1,X_{p+1})$, there is a production rule $r \Rightarrow r_1r_2...r_p$.
\item For each $r\in\mathcal{R}_1$, there is a production rule $r\Rightarrow \overline{r}$.
\end{itemize}
The elements of $(\mathcal{R}\,{\setminus}\,\mathcal{R}_1) \cup \{\overline{r} \,|\, r\,{\in}\, \mathcal{R}_1\}$ are terminal symbols, those in $\mathcal{R}_1$ are non-terminal symbols, and $r$ is the start symbol. We write $L_r$ for the corresponding language.
\begin{proposition}\label{propNoCFGrulebases}
Let $\mathcal{P}$ be a set of closed path rules and suppose that there exists a rule graph $\mathcal{H}$ for $\mathcal{P}$. Let $\mathcal{R}_1$ be the set of relations that appear in the head of some rule in $\mathcal{P}$. It holds that the language $L_r$ is regular for every $r\in\mathcal{R}_1$. 
\end{proposition}
\noindent This shows that we cannot capture arbitrary rule bases using rule graphs. For instance, for the rule base from Example \ref{exRuleGraph3}, we have $L_{r_2} = \{r_1^{l} \overline{r}_2 r_1^{l} \,|\, l\in \mathbb{N}\setminus\{0\}\}$, where we write $x^{l}$ for the string that consists of $l$ repetitions of $x$. It is well-known that this language is not regular. %, hence it follows from Proposition \ref{propNoCFGrulebases} that no rule graph exists for this rule base. 

Following this negative result, we now establish two important positive results. First, in Section \ref{secLeftRegular}, inspired by regular grammars, we introduce a special class of rule bases with cyclic dependencies for which a rule graph is guaranteed to exist. Second, in Section \ref{secBounded}, we focus on the practically important setting of bounded inference: since GNNs use a fixed number of layers in practice, what mostly matters is what can be derived in a bounded number of steps. It turns out that if we only care about such inferences, we can capture arbitrary sets of closed path rules.

%***********************
\subsection{Left-Regular Rule Bases}\label{secLeftRegular}
To show that many rule bases with cyclic dependencies can still be captured, we consider the following notion of a left-regular rule base, inspired by left-regular grammars.

\begin{definition}\label{defRegularRuleBase}
Let $\mathcal{P}$ be a rule base. Let $\mathcal{R}_1$ be the set of relations that appear in the head of a rule from $\mathcal{P}$. We call $\mathcal{P}$ left-regular if every rule is of the following form:
\begin{align}
r_1(X,Y)\wedge r_2(Y,Z) \rightarrow r_3(X,Z)\label{eqLeftRegularRule}
\end{align}
such that $r_2\notin \mathcal{R}_1$. 
\end{definition}
\noindent While Definition \ref{defRegularRuleBase} only considers rules with two relations in the body, rules with more than two atoms can straightforwardly be simulated by introducing fresh relations. 
The following result shows that left-regular rule bases can always be faithfully captured by a {\modelName} model.
\begin{proposition}\label{propLeftRegularRuleGraph}
For any left-regular set of closed path rules $\mathcal{P}$, there exists a rule graph for $\mathcal{P}$.
\end{proposition}
\begin{proof}(Sketch)
Given a left-regular rule base $\mathcal{P}$, we construct the corresponding rule graph $\mathcal{H}$ as follows.
%, where all nodes are of the form $n_R$ with $R\subseteq \mathcal{R}$:
\begin{enumerate}
\item We add the node $n_0$. 
\item For each relation $r\in \mathcal{R}$, we add a node $n_r$, and we connect $n_0$ to $n_r$ with an $r$-edge. 
\item For each rule of the form \eqref{eqLeftRegularRule}, we add an $r_2$-edge from $n_{r_1}$ to $n_{r_3}$.
\item For each node $n$ with multiple incoming $r$-edges for some $r\in\mathcal{R}$, we do the following. Let $\sharp(r,n)$ be the number of incoming $r$-edges for node $n$. Let $p = \max_{r\in\mathcal{R}}\sharp(r,n)$. We create fresh nodes $n_1,...,n_{p-1}$ and add \textit{eq}-edges from $n_i$ to $n_{i-1}$ ($i\in\{1,...,p-1\}$), where we define $n_0=n$. Let $r\in \mathcal{R}$ be such that $\sharp(r,n)>1$. Let $n'_0,...,n'_q$ be the nodes with an $r$-link to $n$; then we have $q\leq p-1$. For each $i\in\{1,...,q\}$ we replace the edge from $n'_i$ to $n$ by an edge from $n'_i$ to $n_i$.
\end{enumerate}
The correctness of this process is shown in the appendix.
\end{proof}
%
%

% It is also easy to verify that any rule base without cyclic dependencies (i.e.\ where we can order the rules such that any relation in the head of a rule does not occur in the body of rules that come before it) can always be rewritten into a semantically equivalent left-regular rule base by introducing fresh relations. 
% \begin{example}
% Let $\mathcal{P} = \{ r_1(X,Y) \wedge r_2(Y,Z) \wedge r_3(Z,U) \rightarrow r_4(X,U), r_1(X,Y)\wedge r_4(Y,Z)\rightarrow r_5(X,Z)\}$. We can rewrite this rule base as $\mathcal{P}' = \{ r_1(X,Y) \wedge r_2(Y,Z) \rightarrow s_1(X,Z), s_1(X,Y) \wedge r_3(Y,Z) \rightarrow r_4(X,Z), r_1(X,Y)\wedge r_4(Y,Z)\rightarrow r_5(X,Z)\}$.
% \end{example}
% 

%**************************************************************************
\subsection{Bounded Inference}\label{secBounded}
In practice, the GNN can only carry out a finite number of inference steps. Rather than requiring that the resulting embeddings capture all triples that can be inferred from $\mathcal{P}\cup \mathcal{G}$, it is thus natural to merely require that the embeddings capture all triples that can be inferred using a bounded number of inference steps. 
%of the form \eqref{eqLeftRegularRule}, but we no longer require that $r_2\notin\mathcal{R}_1$. 
We know from Proposition \ref{propNoCFGrulebases} that it is not always possible to construct a  rule graph for a given rule base $\mathcal{P}$. To address this, we now weaken the notion of a rule graph, aiming to capture reasoning up to a fixed number of inference steps. In the following, we will assume that $\mathcal{P}$ only contains rules with two relations in the body.
%(i.e.\ as in \eqref{eqLeftRegularRule} but without imposing the requirement that $r_2\notin\mathcal{R}_1$). 
Note that we can assume this w.l.o.g.\ as any set of closed path rules can be converted into such a format by introducing fresh relations.

Let us write $\mathcal{P}\cup\mathcal{G}\models_m (e,r,f)$ to denote that 
$(e,r,f)$ can be derived from $\mathcal{P}\cup\mathcal{G}$ in $m$ steps.  More precisely, we have $\mathcal{P}\cup\mathcal{G}\models_0 (e,r,f)$ iff $(e,r,f)\in \mathcal{G}$. Furthermore, we have $\mathcal{P}\cup\mathcal{G}\models_m (e,r,f)$, for $m>0$, iff $\mathcal{P}\cup\mathcal{G}\models_{m-1} (e,r,f)$ or there is a rule $r_1(X_1,X_2)\wedge r_2(X_2,X_3)\rightarrow r(X_1,X_3)$ in $\mathcal{P}$ and an entity $g\in\mathcal{E}$ such that $\mathcal{P}\cup\mathcal{G}\models_{m_1} (e,r_1,g)$ and $\mathcal{P}\cup\mathcal{G}\models_{m_2} (g,r_2,f)$, with $m=m_1+m_2+1$.

\begin{definition}
Let $m\in\mathbb{N}$. We call $\mathcal{H}$ an \emph{$m$-bounded rule graph} for $\mathcal{P}$ if $\mathcal{H}$ satisfies conditions (R1)--(R3) as well as the following weakening of (R4):
\begin{description}
\item[\textbf{(R4m)}] Suppose that for every $r$-edge, there is a path connecting the same nodes whose \textit{eq}-reduced type belongs to $\{(r_{11};...;r_{1p_1}),\allowbreak ...,(r_{q1};...;r_{qp_q})\}$, with $p_1,...,p_q\leq m+1$. Then there is some $i\in\{1,...,q\}$ such that that $\mathcal{P}\models_m r_{i1}(X_1,X_2)\wedge ... \wedge r_{ip_i}(X_{p_i},X_{p_{i+1}})\rightarrow r(X_1,X_{p_{i+1}})$.
\end{description}
\end{definition}

\noindent Given an $m$-bounded rule graph, we can again construct a corresponding {\modelName} model using \eqref{eqDefGNNfromRuleGraph}.
Moreover, Proposition \ref{propRuleGraphCompleteness} remains valid for $m$-bounded rule graphs. Proposition \ref{propProbablyNoFalsePositives} can be weakened as follows.

\begin{proposition}\label{propProbablyNoFalsePositivesBounded}
Let $\mathcal{P}$ be a set of closed path rules $\mathcal{G}$ an \textit{eq}-complete knowledge graph.
Let $\mathcal{H}$ be an $m$-bounded rule graph for $\mathcal{P}$ and let $\mathbf{Z_e^{(l)}}$ be the entity representations that are learned by the GNN,  for the corresponding {\modelName} model.  For any $\varepsilon>0$, there exists some $k_0\in \mathbb{N}$ such that, when $k\geq k_0$, for any $i\leq m+1$ and $(a,r,b)\in\mathcal{E}\times\mathcal{R}\times\mathcal{E}$ such that $\mathcal{P}\cup\mathcal{G}\not\models_m (a,r,b)$, we have
\begin{align*}
\textit{Pr}[ \mathbf{B_r}\mathbf{Z_a^{(i)}}\preceq \mathbf{Z_b^{(i)}}] \leq \varepsilon
\end{align*}
\end{proposition}
\noindent We can again show that there always exists an initialisation of the entity embeddings such that the triple $(a,r,b)$ is not captured by the resulting entity embeddings; we omit the details.
Crucially, we have the following result, showing that $m$-bounded rule graphs always exist, and hence that {\modelName} models can correctly capture bounded reasoning for arbitrary sets of closed path rules.
\begin{proposition}\label{propAlwaysMboundedRuleGraph}
For any set of closed path rules $\mathcal{P}$, there exists an $m$-bounded rule graph for $\mathcal{P}$.
\end{proposition}
\begin{proof}(Sketch)
Given a set of closed path rules $\mathcal{P}$ we can construct an $m$-bounded rule graph as follows.
\begin{enumerate}
\item We add the node $n_0$. 
\item For each relation $r\in \mathcal{R}$, we add a node $n_r$, and we connect $n_0$ to $n_r$ with an $r$-edge. 
\item We repeat the following until convergence. Let  $r\in\mathcal{R}$ and assume there is an $r$-edge from $n$ to $n'$. Let $r_1(X,Y)\wedge r_2(Y,Z)\rightarrow r(X,Z)$ be a rule from $\mathcal{P}$ and suppose that there is no $r_1;r_2$ path connecting $n$ and $n'$. Suppose furthermore that the edge $(n,n')$ is on some path from $n_0$ to a node $n_{r'}$, with $r'\in \mathcal{R}$, whose length is at most $m$. We add a fresh node $n''$ to the rule graph, an $r_1$-edge from $n$ to $n''$, and an $r_2$-edge from $n''$ to $n'$.
\item For each $r\in\mathcal{R}$ and $r$-edge $(n,n')$ such that for some rule $r_1(X,Y)\wedge r_2(Y,Z)\rightarrow r(X,Z)$ from $\mathcal{P}$ there is no $r_1;r_2$ path connecting $n$ and $n'$, we do the following:
\begin{enumerate}
\item We add a fresh node $n''$, an $r_1$-edge from $n$ to $n''$ and an $r_2$-edge from $n''$ to $n'$.
\item We repeat the following until convergence. For each $r'$-edge from $n$ to $n''$ and each rule $r_1'(X,Y)\wedge r_2'(Y,Z)\rightarrow r'(X,Z)$ from $\mathcal{P}$, we add an $r_1'$ edge from $n$ to $n''$ and an $r'_2$-loop to $n''$ (if no such edges/loops exist yet).
\item We repeat the following until convergence. For each $r'$-edge from $n''$ to $n'$ and each rule $r_1'(X,Y)\wedge r_2'(Y,Z)\rightarrow r'(X,Z)$ from $\mathcal{P}$, we add an $r_1'$-loop to $n''$ and an $r_2'$-edge from $n''$ to $n'$ (if no such edges/loops exist yet).
\item We repeat the following until convergence. For each $r'$-loop at $n''$, and each rule $r_1'(X,Y)\wedge r_2'(Y,Z)\rightarrow r'(X,Z)$ from $\mathcal{P}$, we add an $r_1'$-loop and an $r_2'$-loop to $n''$ (if no such loops exist yet).
\end{enumerate}
\item For each node $n$ with multiple incoming $r$-edges for one or more relations from $\mathcal{R}$, we do the following. Let $\sharp(r,n)$ be the number of incoming $r$-edges for node $n$. Let $p = \max_{r\in\mathcal{R}}\sharp(r,n)$. We create fresh nodes $n_1,...,n_{p-1}$ and add \textit{eq}-edges from $n_i$ to $n_{i-1}$ ($i\in\{1,...,p-1\}$), where we define $n_0=n$. Let $r\in \mathcal{R}$ be such that $\sharp(r,n)>1$. Let $n'_0,...,n'_q$ be the nodes with an $r$-link to $n$; then we have $q\leq p-1$. For each $i\in\{1,...,q\}$ we replace the edge from $n'_i$ to $n$ by an edge from $n'_i$ to $n_i$.
\end{enumerate}
The correctness of this process is shown in the appendix.
\end{proof}

\begin{table*}[t]
\centering
\footnotesize
\setlength\tabcolsep{3.4pt}
\begin{tabular}{llccccccccccccc}
\toprule
                                       &
                                       & \multicolumn{4}{c}{\textbf{FB15k-237}} & \multicolumn{4}{c}{\textbf{WN18RR}}    & \multicolumn{4}{c}{\textbf{NELL-995}}  \\
                                       \cmidrule(lr){3-6} \cmidrule(lr){7-10} \cmidrule(lr){11-14}   
                                        &                      
                                       & v1    & v2    & v3    & v4    & v1    & v2    & v3    & v4    & v1    & v2    & v3    & v4    \\
 \multirow{4}{*}{\rotatebox[origin=c]{90}{GNN} }     & 
CoMPILE & 0.676 & 0.829 & 0.846 & 0.874 & 0.836 & 0.798 & 0.606 & 0.754 & 0.583 & 0.938 & 0.927 & 0.751 \\
                                        & 
                                       GraIL  & 0.642 & 0.818 & 0.828 & 0.893 & 0.825 & 0.787 & 0.584 & 0.734 & 0.595 & 0.933 & 0.914 & 0.732 \\
                                        & 
                                       NBFNet                     & \textbf{0.845} & 0.949 & 0.946 & 0.947 & \textbf{0.946} & \textbf{0.897} & \textbf{0.904} & \textbf{0.889} & 0.644 & \textbf{0.953} & \textbf{0.967} & \textbf{0.928} \\
                                    & MorsE (RotatE)                  &  0.832  & \textbf{0.957} & \textbf{0.957} & \textbf{0.959} & 0.841 & 0.815 & 0.709 & 0.796 & 0.652 & 0.807 & 0.877 & 0.534\\
                                        \midrule
 \multirow{2}{*}{\rotatebox[origin=c]{90}{Rule} } & 
                                       RuleN                      & 0.498 & 0.778 & 0.877 & 0.856 & 0.809 & 0.782 & 0.534 & 0.716 & 0.535 & 0.818 & 0.773 & 0.614 \\
 & 
                                       AnyBURL                    & 0.604 & 0.823 & 0.847 & 0.849 & 0.867 & 0.828 & 0.656 & 0.796 & \textbf{0.683} &0.835 & 0.798 & 0.652 \\
\midrule
\multirow{3}{*}{\rotatebox[origin=c]{90}{Diff-R}}                                        & 
                                       DRUM                       & 0.529 & 0.587 & 0.529 & 0.559 & 0.744 & 0.689 & 0.462 & 0.671 & 0.194 & 0.786 & 0.827 & 0.806 \\
 & 
                                       Neural-LP                  & 0.529 & 0.589 & 0.529 & 0.559 & 0.744 & 0.689 & 0.462 & 0.671 & 0.408 & 0.787 & 0.827 & 0.806 \\
                                        & 
\modelName                  & 0.747 & 0.885 & 0.903 & 0.918 & 0.710 & 0.729 & 0.602 & 0.694 & 0.638  & 0.861 & 0.882 & 0.812
\\
                                       \bottomrule
\end{tabular}
\caption{Hits@10 for 50 negative samples on inductive KGC split by method type (GNN-based vs.\ rule-based vs.\ differentiable rule-based).  AnyBURL and NBFNet results were obtained from \citet{DBLP:conf/coling/AnilGIS24}; Neural-LP, DRUM, RuleN, and GraIL results are from \citet{DBLP:conf/icml/TeruDH20}; CoMPILE results are from \citet{DBLP:conf/aaai/MaiZY021}; and MorsE results are from \citet{DBLP:conf/sigir/Chen0ZZYXC22}.}
\label{tab:InductiveRes}
\end{table*}

\begin{table*}[t]
\centering
\footnotesize
\setlength\tabcolsep{3.6pt}
\begin{tabular}{lcccccccccccc}
\toprule
                                  & \multicolumn{4}{c}{\textbf{FB15k-237}}                                                                        & \multicolumn{4}{c}{\textbf{WN18RR}}                                                                           & \multicolumn{4}{c}{\textbf{NELL-995}}                                                                         \\  \cmidrule(lr){2-5} \cmidrule(lr){6-9} \cmidrule(lr){10-13} 
                                  & v1                        & v2                        & v3                        & v4                        & v1                        & v2                        & v3                        & v4                        & v1                        & v2                        & v3                        & v4                        \\ \midrule
\modelName$^2$     & \multicolumn{1}{r}{0.304} & \multicolumn{1}{r}{0.569} & \multicolumn{1}{r}{0.385} & \multicolumn{1}{r}{0.916} & \multicolumn{1}{r}{0.293} & \multicolumn{1}{r}{0.309} & \multicolumn{1}{r}{0.155} & \multicolumn{1}{r}{0.270} & \multicolumn{1}{r}{0.488} & \multicolumn{1}{r}{0.558} & \multicolumn{1}{r}{0.334} & \multicolumn{1}{r}{0.370} \\
\modelName$_{\sf nL}$ & \multicolumn{1}{r}{0.744} & \multicolumn{1}{r}{0.890} & \multicolumn{1}{r}{0.903} & \multicolumn{1}{r}{0.917} & \multicolumn{1}{r}{0.698} & \multicolumn{1}{r}{0.685} & \multicolumn{1}{r}{0.618} & \multicolumn{1}{r}{0.682} & \multicolumn{1}{r}{0.627} & \multicolumn{1}{r}{0.738} & \multicolumn{1}{r}{0.886} & \multicolumn{1}{r}{0.815} \\
\modelName         & 0.747                     & 0.885                     & 0.903                     & 0.918                     & 0.710                     & 0.729                     & 0.602                     & 0.694                     & 0.638                     & 0.861                     & 0.882                     & 0.812                     \\ \bottomrule
\end{tabular}
\caption{Hits@10 for 50 negative samples on inductive KGC for each ablation of \modelName.}
\label{tab:ablationRes}
\end{table*}

%**************************************************************************
\section{Beyond Closed Path Rules}
Thus far we have only focused on sets of closed path rules, motivated by their importance for KG completion, and the known limitations of existing region-based models when it comes to capturing such rules. Two other important types of rules are hierarchy and intersection rules, which are respectively of the following form:
\begin{align*}
r_1(X,Y) &\rightarrow r_2(X,Y)\\
r_1(X,Y) \wedge r_2(X,Y) &\rightarrow r_3(X,Y)
\end{align*}
Existing models \citep{DBLP:conf/nips/AbboudCLS20,DBLP:conf/iclr/0002S23,DBLP:conf/ijcai/CharpenayS24} are capable of capturing arbitrary sets of such rules. We now show that the same is true for {\modelName} (with high probability).

\begin{proposition}\label{propHierarchyIntersection}
Let $\mathcal{P}$ be a set of hierarchy and intersection rules. There exists a {\modelName} model such that for every knowledge graph $\mathcal{G}$ the following conditions are satisfied:
\begin{enumerate}
\item Suppose that $\mathcal{P}\cup\mathcal{G} \models (a,r,b)$ and let $\tau$ be an entity embedding such that $(\tau,\eta)$ captures every triple in $\mathcal{G}$. It holds that $(\tau,\eta)$ captures $(a,r,b)$.
\item For any $\varepsilon>0$ there is a $k_0\in\mathbb{N}$ such that, when $k\geq k_0$, for any $m\in \mathbb{N}$ and $(a,r,b)\in\mathcal{E}\times\mathcal{R}\times\mathcal{E}$ such that $\mathcal{P}\cup\mathcal{G} \not\models (a,r,b)$, we have $\textit{Pr}[ \mathbf{B_r}\mathbf{Z_a^{(m)}}\preceq \mathbf{Z_b^{(m)}}] \leq \varepsilon$, where $\mathbf{Z_e^{(m)}}$ are the entity representations that are learned by the GNN \eqref{eqGNNformB}.
\end{enumerate}
\end{proposition}

\noindent Whether it is possible to capture bounded inference for arbitrary sets of closed path rules, intersection rules, and hierarchy rules remains an open question for future work. 
In the basic formulation of {\modelName}, we cannot model inverse relations, which also means that we cannot constrain relations to be symmetric. However, in practice, for each triple $(e,r,f)$ in the KG, we add an inverse triple $(f,r^{\text{inv}},e)$. Each triple thus induces two constraints $\mathbf{B_r}\mathbf{Z_e}\preceq \mathbf{Z_f}$ and $\mathbf{B_{r^{\text{inv}}}}\mathbf{Z_f}\preceq \mathbf{Z_e}$. The fact that $r$ is a symmetric relation can then be straightforwardly captured by requiring $\mathbf{B_r}=\mathbf{B_{r^{\text{inv}}}}$.
Disjointness constraints of the form $r_1(X,Y)\wedge r_2(X,Y)\rightarrow \bot$, indicating that $r_1$ and $r_2$ can never hold together, have also been studied in the context of region-based embeddings. {\modelName} does not have a mechanism to capture such constraints, but this can be addressed by adding a bias term, associating each triple $(e,r,f)$ with constraints of the form $\mathbf{B_r}\mathbf{Z_e} + \mathbf{c_r}\preceq \mathbf{Z_f}$ and $\mathbf{B_{r^{\text{inv}}}}\mathbf{Z_f} + \mathbf{c_r^{\text{inv}}}\preceq \mathbf{Z_e}$. 

%**************************************************************************
\section{Empirical Evaluation}
We complement our theoretical results with an empirical evaluation, focused on showing that suitable model parameters can be effectively learned. In particular, we want to assess whether the model is simple enough to avoid overfitting, and whether it can compete with other (differentiable) rule learning methods. We focus on \emph{inductive} KG completion, as the need to capture reasoning patterns is intuitively more important for this setting compared to the traditional (i.e.\ transductive) setting.\footnote{The code for replicating our experiments is available at:\\ \url{https://github.com/AleksVap/RESHUFFLE}.} 

\paragraph{Model Details}
We learn a soft approximation of the matrices $\mathbf{B_r}$. Specifically, we learn each row $i$ of $\mathbf{B_r}$ by selecting the first $\ell$ coordinates of a vector $\mathsf{softmax}(b^r_{i,1},...,b^r_{i,\ell+1})$, with $b^r_{i,1},...,b^r_{i,\ell+1}$ learnable parameters. Note that we need $\ell+1$ parameters to allow some rows to be all 0s, which we empirically found to be important. The number of parameters per relation is thus quadratic in $\ell$. However, due to the use of the softmax operation, these representations can still be learned effectively \citep{DBLP:conf/iclr/LavoieTSVNKC23}.\footnote{We experimented with a number of strategies for imposing sparsity, but were not able to outperform the softmax formulation.}

To initialise the entity embeddings, we set each coordinate to 0 or 1, with 50\% probability. To train the model, we use the following scoring function for a given triple $(e,r,f)$:
\begin{align*}
s(e,r,f) = - \|\mathsf{ReLU}(\mathbf{B_r}\,\mathbf{Z_e^{(m)}}-\mathbf{Z_f^{(m)}})\|_2 
\end{align*}
where $m$ denotes the number of GNN layers. Note that $s(e,r,f)=0$ reaches its maximal value of 0 iff $\mathbf{B_r}\mathbf{Z_e^{(m)}}\preceq \mathbf{Z_f^{(m)}}$. For each triple $(e,r,f)$ in the given KG $\mathcal{G}$, we add an inverse triple $(f,r_{\textit{inv}},e)$ to $\mathcal{G}$. For each entity $e$, we also add the triple $(e,\textit{eq},e)$ to $\mathcal{G}$. Following the literature \citep{DBLP:conf/icml/TeruDH20,DBLP:conf/nips/ZhuZXT21}, \modelName's training process uses negative sampling under the partial completeness assumption (PCA) \citep{DBLP:conf/www/GalarragaTHS13}, i.e., for each training triple $(e, r, f) \in \mathcal{G}$, $N$ triples (negative samples) are created by replacing $e$ or $f$ in $(e, r, f)$ by randomly sampled entities $e', f' \in \mathcal{E}$.  
To train \modelName, we minimise the margin ranking loss, defined as follows: 
\begin{align} L(e, r, f) = \sum^N_{i=1} \max(0, s(e_i', r, f_i') - s(e, r, f) + \lambda) \label{eq:margin_ranking_loss}
\end{align}
where $(e_i', r, f_i')$ is the i\textsuperscript{th} negative sample and $\lambda>0$ is a hyperparameter, called the margin. Intuitively, the margin ranking loss pushes scores of true triples (i.e., those within the training graph) to be larger by at least $\lambda$ than the scores of triples that are likely false (i.e., negative samples).

%\paragraph{Datasets} 
\paragraph{Experimental Setup} 
We evaluate \modelName\ on the inductive knowledge graph completion (KGC) benchmark that was derived by \citet{DBLP:conf/icml/TeruDH20} from FB15k-237, WN18RR, and NELL-995. For each of these KGs, four different datasets were obtained, named v1 to v4, leading to a total of twelve datasets. %The benchmark  tus consists of twelve datasets: FB15k-237 v1-4, WN18RR v1-4, and NELL-995 v1-4.
Each dataset consists of two disjoint graphs, a training graph $\mathcal{G}_{\textit{Train}}$ and a testing graph $\mathcal{G}_{\textit{Test}}$. %, which are sampled from the original dataset as follows. The \emph{training graph} $\mathcal{G}_{\textit{Train}}$ was obtained by randomly sampling different numbers of entities and selecting their $k$-hop neighbourhoods. Next, to construct a disjoint \emph{testing graph} $\mathcal{G}_{\textit{Test}}$, the entities of $\mathcal{G}_{\textit{Train}}$ were removed from the initial graph, and the same sampling procedure was repeated. 
Both of these graphs are split into a train set ($80\%$), validation set ($10\%$), and test set ($10\%$), leading to a total of six graphs per dataset. %Thus, the three inductive benchmarks consist in total of twelve datasets:  Furthermore, each of these datasets consists of six graphs: the train, validation, and test splits of $\mathcal{G}_{\textit{Train}}$ and $\mathcal{G}_{\textit{Test}}$.  
%The appendix provide additional information about these benchmarks.
%
%\paragraph{Experimental Setup} 
Following \citet{DBLP:conf/icml/TeruDH20}, we train \modelName\ on the train split of $\mathcal{G}_{\textit{Train}}$, tune our model's hyperparameters on the validation split of $\mathcal{G}_{\textit{Train}}$, and finally evaluate the  best model on the test split of $\mathcal{G}_{\textit{Test}}$. 
%As discussed by \citet{DBLP:conf/coling/AnilGIS24}, some approaches in the literature have been evaluated in different ways, e.g.\ by tuning hyperparameters on the validation split of $\mathcal{G}_{\textit{Test}}$, and their reported results are thus not directly comparable.  %We list additional information on the experimental setup, including the details of the implementation and evaluation protocol, and the used hyperparameters in the appendix. 
%
To account for small performance fluctuations, we repeat our experiments three times and report \modelName's average performance.\footnote{Results for all seeds and the resulting standard deviations are provided in the appendix.} We select the hyperparameter configuration with the highest Hits@10 score on the validation split of $\mathcal{G}_{\textit{Train}}$. Further details on hyperparameter tuning can be found in the appendix. In accordance with \citet{DBLP:conf/icml/TeruDH20}, we evaluate \modelName's test performance on $50$ negatively sampled entities per triple of the test split of $\mathcal{G}_{\textit{Test}}$ and report the Hits@10 scores. We list further details about the experimental setup in the appendix.

%\paragraph{Baselines} 
Our GNN model acts as a kind of differentiable rule base. We therefore compare \modelName\ to standard approaches for differentiable rule learning: Neural-LP \citep{DBLP:conf/nips/YangYC17} and DRUM \citep{DBLP:conf/nips/SadeghianADW19}. We also compare our method to two classical rule learning methods:  RuleN \citep{DBLP:conf/semweb/MeilickeFWRGS18} and AnyBURL \citep{DBLP:conf/ijcai/MeilickeCRS19}. Finally, we include a comparison with GNN-based approaches: CoMPILE \citep{DBLP:conf/aaai/MaiZY021}, GraIL \citep{DBLP:conf/icml/TeruDH20}, NBFNet \citep{DBLP:conf/nips/ZhuZXT21}, and MorsE \citep{DBLP:conf/sigir/Chen0ZZYXC22}.

%\paragraph{Training Details} 

\paragraph{Results} 
The results in Table~\ref{tab:InductiveRes} reveal that \modelName\ consistently outperforms the differentiable rule learners DRUM and Neural-LP, often by a significant margin (with WN18RR-v1 as the only exception). Compared to the traditional rule learners,  \modelName\ performs clearly better on FB15k-237 and NELL-995 (apart from v1) but underperforms on the WN18RR benchmarks.  \citet{DBLP:conf/coling/AnilGIS24} found that the kinds of rules which are needed for WN18RR are much noisier compared to those than those which are needed for FB15k-237 and NELL-995. Our use of ordering constraints may be less suitable for such cases. Finally, compared to the GNN-based methods, \modelName\ outperforms CoMPILE and GraIL on FB15k-237 and NELL-995 v1 and v4 while again (mostly) underperforming on WN18RR. 
Despite the promising results compared to (differentiable) rule learners, \modelName\ is not competitive against state-of-the-art models such as NBFNet and MorsE. This finding is compatible with the analysis from \citet{DBLP:conf/coling/AnilGIS24}, which suggests that achieving state-of-the-art performance requires going beyond rule-based reasoning. 

%However, the practical significance of \modelName\ stems from the fact that queries can be evaluated in constant time (if we treat the dimensionality of the entity embeddings as a constant), once the embeddings have been computed, whereas the latter only requires a single forward pass of the GNN. In contrast, NBFNet requires a forward pass of a GNN for each query, something which remains true for more efficient variants such as A$^*$Net \citep{DBLP:conf/nips/ZhuYGX0G023}. Furthermore, the monotonic nature of our GNN means that embeddings can be straightforwardly updated when new information is added to the KG.

\paragraph{Ablation Analysis}
We consider two variants of our model: $(i)$ \modelName$_{\sf nL}$, which does not add a self-loop relation to the KG (i.e.\ triples of the form $(e,\textit{eq},e)$); and $(ii)$ \modelName$^2$, which allows for more general $\mathbf{B_r}$ matrices. Different from \modelName, which uses the softmax function to learn the rows of $\mathbf{B_r}$, \modelName$^2$ squares the $\mathbf{B_r}$ matrices component-wise, thereby allowing arbitrary positive values.
For a fair comparison, we train each variant with the same hyperparameter values, experimental setup, and evaluation protocol. 
% We expect the following results \modelName$^2$ will overfit and \modelName$_{\sf nL}$ will perform slightly worse as it loses some of the theoretical capabilities proven in Section \ref{} \todo{Add another table at the bottom!}.
The results in Table \ref{tab:ablationRes} show that \modelName\ performs comparable to or better than \modelName$_{\sf nL}$ and dramatically outperforms \modelName$^2$ on all benchmarks. The similar performance of \modelName\ and \modelName$_{\sf nL}$ on most datasets suggests that the self-loop relation only matters in specific cases, which may not occur frequently in some datasets. The poor performance of \modelName$^2$ is as expected, since allowing arbitrary positive parameters makes overfitting the training data more likely.

%**************************************************************************
\section{Conclusions}
The region-based view of KG embeddings makes it possible to formally analyse which inference patterns are captured by a given embedding. An important question, which was left unanswered by previous work, is whether a region-based embedding model can be found which is capable of capturing arbitrary sets of closed path rules, while still ensuring that embeddings can be learned effectively in practice. In this context, we proposed a novel approach based on ordering constraints between reshuffled entity embeddings. This model, called {\modelName}, was chosen because it allows us to escape the limitations of coordinate-wise approaches while otherwise remaining as simple as possible. We found that {\modelName} has several interesting properties. Most significantly, we showed that bounded reasoning with arbitrary sets of closed path rules can be faithfully captured. We also revealed special cases where exact reasoning is possible, 
%namely rule bases without cyclic dependencies and left-regular rule bases. These classes of rule bases 
which go significantly beyond what is (known to be) possible with existing region based models. %Finally, we have shown that sets of intersection and hierarchy rules can also be captured, matching what is known about the expressivity of existing region-based models in this respect.

Empirically, we found our approach to be competitive with (differentiable) rule learners, while underperforming the state-of-the-art more generally. This latter finding reflects the fact that (differentiable) rule based methods are less suitable when we need to weigh different pieces of weak evidence. In such cases, when further evidence becomes available, we may want to revise earlier assumptions, which is not possible with {\modelName}. Developing effective models that can provably simulate non-monotonic (or probabilistic) reasoning thus remains as an important open challenge. %However, while important areas for future work remain, we view {\modelName} as an 
Another interesting direction for future work would be to extend our model to relations of higher arity. 

%% The file kr.bst is a bibliography style file for BibTeX 0.99c
\section*{Acknowledgments}
This work was supported by the Vienna Science and Technology Fund
[10.47379/VRG18013, 10.47379/NXT22018, 10.47379/ ICT2201], the
Austrian Science Fund [10.55776/COE12], and EPSRC grant EP/W003309/1.

\bibliographystyle{kr}
\bibliography{refs}

\newpage
\appendix

\section{Constructing Models from Rule Graphs}
Let $\mathcal{P}$ be a set of closed path rules and let $\mathcal{H}$ be a corresponding rule graph, satisfying the conditions (R1)--(R4). 
We also assume that a knowledge graph $\mathcal{G}$ is given. We show that the {\modelName} model, which is constructed based on $\mathcal{H}$, faithfully captures $\mathcal{P}$. For the proofs, it will be more convenient to characterise the model in terms of operations on the coordinates of entity embeddings. Specifically, let $Z_i = \{(i-1)k+1,...,(i-1)k+k\}$ and let $N_r\subseteq \{n_1,...,n_{\ell}\}$ be the set of nodes from the rule graph $\mathcal{H}$ which have an incoming edge labelled with $r$. We define:
\begin{align*}
I_r &= \bigcup_{n_i\in N_r} Z_i
\end{align*}
Let $n_i\in N_r$ and let $(n_j,n_i)$ be the unique incoming edge with label $r$. Then we define ($t\in \{1,...,k\}$):
\begin{align*}
\sigma_r((i-1)k + t) &= (j-1)k + t
\end{align*}
Now let us define:
\begin{align*}
\mu_r(e_1,...,e_d) &= (e_1',...,e_d')
\end{align*}
where $e_i'=e_{\sigma_r(i)}$ if $i\in I_r$ and $e_i'=0$ otherwise. Let $\mathbf{e}$ be the entity embedding corresponding to the matrix $\mathbf{Z_e}$. In other words, if we write $z_{ij}$ for the components of $\mathbf{Z_e}$ and $e_i$ for the components of $\mathbf{e}$, then we have $z_{ij}= e_{(i-1)k+j}$. For a matrix $\mathbf{X}=(x_{ij})$, let us write $\textit{flatten}(\mathbf{X})$ for the vector that is obtained by concatenating the rows of $\mathbf{X}$. In particular, $\textit{flatten}(\mathbf{Z_e})=\mathbf{e}$. The following lemma reveals how the model constructed from the rule graph $\mathcal{H}$ can be characterised in terms of entity embeddings.

\begin{lemma}\label{eqEquivalenceRepresentationsEntities}
It holds that $\textit{flatten}(\mathbf{B_r}\mathbf{Z_e}) =\mu_r(\mathbf{e})$.
\end{lemma}
\begin{proof}
Let us write $\textit{flatten}(\mathbf{B_r}\mathbf{Z_e})=(x_1,...,x_d)$, $\mu_r(\mathbf{e})= (y_1,...,y_d)$ and $\mathbf{e}=(e_1,...,e_d)$. Let $i\in \{1,...,\ell\}$. Let us first assume that $n_i$ does not have any incoming edges in $\mathcal{H}$ which are labelled with $r$. In that case, row $i$ of $\mathbf{B_r}$ consists only of 0s and we have $x_{(i-1)k+1}=...=x_{(i-1)k+k}=0$. Similarly, we then also have $(i-1)k+j\notin I_r$ for $j\in\{1,...,k\}$ and thus $y_{(i-1)k+1}=...=y_{(i-1)k+k}=0$.
Now assume that there is an edge from $n_j$ to $n_i$ which is labelled with $r$. Then we have that row $i$ of $\mathbf{B_r}$ is a one-hot vector with 1 at position $j$. Accordingly, we have $x_{(i-1)k+t}=e_{(j-1)k+t}$ for $t\in\{1,...,k\}$. Accordingly we then have $\sigma_r((i-1)k+t)=(j-1)k+t$ and thus $y_{(i-1)k+t}=e_{(j-1)k+t}$. 
\end{proof}

For a sequence of relations $r_1,...,r_p$, we define $\mu_{r_1;...;r_p}$ as follows.  We define $\mu_{r_1;...;r_p}(x_1,...,x_d)$ $=(y_1,...,y_d)$, where ($i \in \{1,...,\ell\}$,  $t\in\{1,...,k\}$):
\begin{align*}
y_{(i-1)k +t} = 
\begin{cases}
x_{(j-1)k + t} & \text{if there is an $r_1;...;r_p$ path }\\
& \text{from $n_j$ to $n_i$}\\
0 & \text{otherwise}
\end{cases}
\end{align*}
Note that if there is an $r_1;...;r_k$ path arriving at node $n_i$ in the rule graph, it has to be unique, given that each node has at most one incoming edge of a given type. In the following, we will also use $I_{r_1;...;r_p}$, defined as follows: 
\begin{align*}
&I_{r_1;...;r_p} \\
&\quad = \{(i-1)k + t \,|\, \text{there is an $r_1;...;r_p$ path ending in $n_i$}\}
\end{align*}
We have the following result.
\begin{lemma}\label{lemmaCompositeMu}
For $r_1,...,r_p\in \mathcal{R}$ we have
$$
\mu_{r_1;...;r_p}(x_1,...,x_d) = \mu_{r_p}(...\mu_{r_1}(x_1,...,x_d)...)
$$
\end{lemma}
\begin{proof}
It is sufficient to show 
$$
\mu_{r_1;...;r_p}(x_1,...,x_d) = \mu_{r_p}(\mu_{r_1;...;r_{p-1}}(x_1,...,x_d))
$$
We have $\mu_{r_1;...;r_{p-1}}(x_1,...,x_d)=(y_1,...,y_d)$, with 
\begin{align*}
y_{(i-1)k +t} = 
\begin{cases}
x_{(j-1)k + t} & \text{if there is an $r_1;...;r_{p-1}$ path}\\
&\text{from $n_j$ to $n_i$}\\
0 & \text{otherwise}
\end{cases}
\end{align*}
We furthermore have $\mu_{r_p}(y_1,...,y_d)=(z_1,...,z_d)$ with 
\begin{align*}
z_{(i-1)k +t} = 
\begin{cases}
y_{(j-1)k + t} & \text{if there is an $r_p$-edge}\\
& \text{from $n_j$ to $n_i$}\\
0 & \text{otherwise}
\end{cases}
\end{align*}
Taking into account the definition of $(y_1,...,y_d)$, we have $y_{(j-1)k + t}\neq 0$ only if there is an $r_1;...;r_{p-1}$ path from some node $n_l$ to the node $n_j$, in which case we have $y_{(j-1)k + t}=x_{(l-1)k + t}$. In other words, we have:
\begin{align*}
z_{(i-1)k +t} = 
\begin{cases}
x_{(l-1)k + t} & \text{if there is an $r_1;...;r_{p-1}$ path  }\\
& \text{from $n_l$ to some $n_j$ and an  }\\
& \text{$r_p$ edge from $n_j$ to $n_i$}\\
0 & \text{otherwise}
\end{cases}
\end{align*}
In other words, we have
\begin{align*}
z_{(i-1)k +t} = 
\begin{cases}
x_{(l-1)k + t} & \text{if there is an $r_1;...;r_p$  path}\\
& \text{from $n_l$ to $n_i$}\\
0 & \text{otherwise}
\end{cases}
\end{align*}
We thus have $(z_1,...,z_d)=\mu_{r_1;...;r_p}(x_1,...,x_d)$.
\end{proof}

\noindent We also have the following result.
\begin{lemma}\label{lemmaPathSummarisingInequality}
Suppose $\mathcal{P}\models r_1(X_1,X_2) \wedge r_2(X_2,X_3) \wedge ... \wedge r_p(X_p,X_{p+1}) \rightarrow r(X_1,X_{p+1})$. 
There exists paths of type $r^1_1;...;r^1_{q_1}$ and $r^2_1;...;r^2_{q_2}$ and ... and $r^l_1;...;r^l_{q_l}$, all of whose \textit{eq}-reduced type is $r_1;...;r_p$, such that
for every embedding $(x_1,...,x_d)$ we have: 
$$
\mu_r(x_1,...,x_d) \preccurlyeq \max_{i=1}^l \mu_{r^i_1;...;r^i_{q_i}}(x_1,...,x_d) 
$$
\end{lemma}
\begin{proof}
This follows immediately from the fact that whenever there is an $r$-edge between two nodes $n$ and $n'$, there must also be a path between these nodes whose  \textit{eq}-reduced type is $r_1;...;r_p$, because of condition (R3).
\end{proof}

\noindent We now show that the constructed {\modelName} model captures all triples that can be inferred from $\mathcal{G}\cup\mathcal{P}$.

\begin{proposition}\label{propRuleGraphCompletenessAppendix}
Let $\mathcal{P}$ be a set of closed path rules and $\mathcal{G}$ a \textit{eq}-complete knowledge graph.
Suppose $\mathcal{P}\cup \mathcal{G} \models (a,r,b)$. Let $\mathcal{H}$ be a rule graph for $\mathcal{P}$ and let $\eta$ be the corresponding {\modelName} model. Let $\tau$ be an entity embedding such that $(\tau,\eta)$ captures every triple in $\mathcal{G}$. It holds that $(\tau,\eta)$ captures the triple $(a,r,b)$.
\end{proposition}
\begin{proof}
Let us write $\tau(e)=\mathbf{Z_e}$ and $\textit{flatten}(\mathbf{Z_e})=\mathbf{e}$.
Because of Lemma \ref{eqEquivalenceRepresentationsEntities}, it is sufficient to show that $\mu_r(\mathbf{a}) \preceq \mathbf{b}$. If $\mathcal{G}$ contains the triple $(a,r,b)$ then the result is trivially satisfied. Otherwise, $\mathcal{P}\cup \mathcal{G} \models r(a,b)$ implies that $\mathcal{P}\models r_1(X_1,X_2) \wedge r_2(X_2,X_3) \wedge ... \wedge r_p(X_p,X_{p+1}) \rightarrow r(X_1,X_{p+1})$, for some $r_1,...,r_p,r\in \mathcal{R}$ such that $\mathcal{G}$ contains triples $(a,r_1,a_2), (a_2,r_2,a_3),..., (a_p,r_p,b)$, for some $a_2,...,a_{p}\in\mathcal{E}$. 
Because $(a,r_1,a_2)\in \mathcal{G}$, by assumption this triple is satisfied by $(\tau,\eta)$ and thus
$$
\mu_{r_1}(\mathbf{a_{1}})\preccurlyeq \mathbf{a_{2}}
$$
Similarly, because $(a_{2},r_{2},a_3)\in \mathcal{G}$, we have $\mu_{r_{2}}(\mathbf{a_{2}})\preccurlyeq \mathbf{a_{3}}$ and thus
$$
\mu_{r_{2}}(\mu_{r_1}(\mathbf{a}))
\preccurlyeq \mu_{r_{2}}(\mathbf{a_2})
\preccurlyeq \mathbf{a_{3}}
$$
In other words, we have
$$
\mu_{r_{1};r_2}(\mathbf{a})\preccurlyeq \mathbf{a_{3}}
$$
Continuing in the same way, we find that
$$
\mu_{r_1;...;r_{p-1};r_p}(\mathbf{a})\preccurlyeq \mathbf{b}
$$
Now consider a path of type $r'_1;...;r_q'$ whose \textit{eq}-reduced type is $r_1;...;r_p$. Then we have that $\mathcal{G}$ contains triples of the form $(a,r'_1,b_2), (b_2,r_2,b_3),..., (b_p,r'_q,b)$. Indeed, the only triples that need to be considered in addition to the triples $(a,r_1,a_2), (a_2,r_2,a_3),..., (a_p,r_p,b)$ are of the form $(a_i,\textit{eq},a_i)$, which we have assumed to belong to $\mathcal{G}$ for every $a_i\in\mathcal{E}$. For every path of type $r'_1;...;r_q'$  whose  \textit{eq}-reduced type is $r_1;...;r_p$, we thus find entirely similarly to before that 
$$
\mu_{r'_1;...;r'_q}(\mathbf{a})\preccurlyeq \mathbf{b}
$$
Because of Lemma \ref{lemmaPathSummarisingInequality}, this implies
$$
\mu_{r}(\mathbf{a})\preccurlyeq \mathbf{b}
$$
\end{proof}

We write $\mathbf{Z_e^{(m)}}$ for the representation of entity $e$ in layer $m$ of the GNN. We similarly use notations such as $\mathbf{e^{(m)}}$. For an entity $e$, we will also use the notation $\textit{emb}_m(e)$ to denote its (flattened) embedding in layer $m$ of the GNN, i.e.\ $\textit{emb}_m(e)=\mathbf{e^{(m)}}$. 
For $e\in \mathcal{E}$, let $\textit{paths}_{\mathcal{G}}(e)$ be the set of all paths in the knowledge graph $\mathcal{G}$ which end in $e$. For a path $\pi$ in $\textit{paths}_{\mathcal{G}}(e)$, we write $\mathit{head}(\pi)$ for the entity where the path starts and $\mathit{rels}(\pi)$ for the corresponding sequence of relations. The following observation follows immediately from the construction of the GNN, together with Lemma \ref{lemmaCompositeMu}.

\begin{lemma}\label{lemmaGNNconvergence}
For any entity $e\in\mathcal{E}$ it holds that
\begin{align*}
\mathbf{e^{(m)}}\preceq\max\Big(\mathbf{e^{(0)}}, \max_{\pi\in \textit{paths}_{\mathcal{G}}(e)} \mu_{\mathit{rels}(\pi)}\big(\textit{emb}_0(\textit{head}(\pi))\big)\Big)
\end{align*}
\end{lemma}

\noindent We will also need the following technical lemma.

\begin{lemma}\label{lemmaNotEntailedZempty}
Suppose $\mathcal{P}\cup\mathcal{G}\not \models (a,r,b)$. Then there is some $i\in\{1,...,\ell\}$ such that:
\begin{itemize}
\item $Z_i\subseteq I_r$; and
\item whenever $\pi\in \textit{paths}_{\mathcal{G}}(b)$ with $\textit{head}(\pi)=a$, it holds that
$
I_{\textit{rels}(\pi)} \cap Z_i = \emptyset
$.
\end{itemize}
\end{lemma}
\begin{proof}
Let us write $\mathcal{Z}_r = \{i\in \{1,...,\ell\} \,|\, Z_i\subseteq I^1_r\}$. Note that $i\in \mathcal{Z}_r$ iff node $n_i$ in $\mathcal{H}$ has an incoming $r$-edge. It thus follows from condition (R1) that $\mathcal{Z}_r\neq \emptyset$. Suppose that for every $i\in \mathcal{Z}_r$, there was some $\pi\in \textit{paths}_{\mathcal{G}}(b)$ with $\textit{head}(\pi)=a$ such that $I_{\textit{rels}(\pi)} \cap Z_i \neq \emptyset$. Let us write $X = \{\textit{rels}(\pi) \,|\,\pi\in \textit{paths}_{\mathcal{G}}(b),  \textit{head}(\pi)=a,  I_{\textit{rels}(\pi)} \cap Z_i \neq \emptyset\}$. We then have that for every $r$-edge in $\mathcal{H}$, there is a path $\tau$ connecting the same nodes, with $\textit{rels}(\tau)\in X$. From Condition (R4), it then follows that $\mathcal{P}\cup\mathcal{G} \models (a,r,b)$, a contradiction.
\end{proof}

\noindent The following result shows that the GNN is unlikely to predict  triples that cannot be inferred from $\mathcal{G}\cup\mathcal{P}$, as long as the embeddings are sufficiently high-dimensional.

\begin{proposition}\label{propProbablyNoFalsePositivesAppendix}
Let $\mathcal{P}$ be a set of closed path rules and $\mathcal{G}$ an \textit{eq}-complete knowledge graph.
Let $\mathcal{H}$ be a rule graph for $\mathcal{P}$ and let $\mathbf{Z_e^{(l)}}$ be the entity representations learned using the GNN for the corresponding {\modelName} model. For any $\varepsilon>0$, there exists some $k_0\in \mathbb{N}$ such that, when $k\geq k_0$, for any  $m\in\mathbb{N}$  and $(a,r,b)\in \mathcal{E}\times\mathcal{R}\times\mathcal{E}$ such that $\mathcal{P}\cup \mathcal{G} \not\models (a,r,b)$, we have
\begin{align*}
\textit{Pr}[ \mathbf{B_r}\mathbf{Z_a^{(m)}}\preceq \mathbf{Z_b^{(m)}}] \leq \varepsilon
\end{align*}
\end{proposition}

\begin{proof}
First, note that because of Lemma \ref{eqEquivalenceRepresentationsEntities}, what we need to show is equivalent to:
\begin{align*}
\textit{Pr}[ \mu_r(\mathbf{a^{(m)}}) \preceq \mathbf{b^{(m)}}] \leq \varepsilon
\end{align*}

Let $(a,b)\in \mathcal{E}\times\mathcal{E}$ be such that $\mathcal{P}\cup\mathcal{G} \not\models (a,r,b)$. 
From Lemma \ref{lemmaNotEntailedZempty}, we know that there is some $i\in \{1,...,\ell\}$ such that $Z_i\subseteq I^1_r$ and
whenever $\pi\in \textit{paths}_{\mathcal{G}}(b)$ with $\textit{head}(\pi)=a$, it holds that $
I_{\textit{rels}(\pi)} \cap Z_i = \emptyset$. The following condition is clearly a necessary requirement for $\mu_r(\mathbf{a^{(m)}}) \preceq \mathbf{b^{(m)}}$:
$$
\forall j\in Z_i\,.\, \mu_r(\mathbf{a^{(m)}}) \preccurlyeq_j  \mathbf{b^{(m)}}
$$
where we write $(x_1,...,x_d)\preccurlyeq_j (y_1,...,y_d)$ for $x_j\leq y_j$. We need in particular also that:
$$
\forall j\in Z_i\,.\, \mu_r(\mathbf{a^{(0)}}) \preccurlyeq_j  \mathbf{b^{(m)}}
$$
Due to Lemma \ref{lemmaGNNconvergence} this is equivalent to requiring that for every $j\in Z_i$ we have:
\begin{align*}
\mu_r(\mathbf{a^{(0)}}) {\preccurlyeq_j}  \max\big(\mathbf{b^{(0)}},\hspace{-5pt} \max_{\pi\in \textit{paths}_{\mathcal{G}}(b)} \mu_{\mathit{rels}(\pi)}\big(\textit{emb}_0(\textit{head}(\pi))\big)\big)
\end{align*}
We can view the coordinates of the input embeddings as random variables. The latter condition is thus equivalent to a condition of the following form:
\begin{align}\label{eqConditionABXrandomvariables}
\forall j\in Z_i\,.\, A^r_j \leq \max(B_j, X^1_j,...,X^p_j)
\end{align}
where $A^r_j$ is the random variable corresponding to the $j$\textsuperscript{th} coordinate of $\mu_r(\mathbf{a^{(0)}})$, $B_j$ is the  $j$\textsuperscript{th} coordinate of $\mathbf{b^{(0)}}$ and $X^1_j,...,X^p_j$ are the random variables corresponding to the $j$\textsuperscript{th} coordinate of the vectors $\mu_{\mathit{rels}(\pi)}\big(\textit{emb}_0(\textit{head}(\pi)))$. 
By construction, we have that the coordinates of different entity embeddings are sampled independently and that there are at least two distinct values that have a non-negative probability of being sampled for each coordinate. This means that there exists some value $\lambda>0$ such that $Pr[A^r_j > B_j]\geq \lambda$ and $Pr[A^r_j > X_j^t]\geq \lambda$ for each $t\in\{1,...,p\}$. Moreover, since we have that whenever $\pi\in \textit{paths}_{\mathcal{G}}(b)$ with $\textit{head}(\pi)=a$ it holds that $
I_{\textit{rels}(\pi)} \cap Z_i = \emptyset$, it follows that the random variable $A^r_j$ is not among $B_j, X^1_j,...,X^p_j$. We thus have:
\begin{align*}
&\textit{Pr}[\forall j\in Z_i\,.\, A^r_j \leq \max(B_j, X^1_j,...,X^p_j)]\\
&\quad\leq\left(1-\lambda^{p+1}\right)^{|Z_i|}\\
&\quad=\left(1-\lambda^{p+1}\right)^{k}\\
&\quad\leq e^{-k\lambda^{p+1}}
\end{align*}
The value of $p$ is upper bounded by $\ell\cdot |\mathcal{E}|$, with $\ell$ the number of nodes in the rule graph.
By choosing $k$ sufficiently large, we can thus make this probability arbitrarily small. In particular:
\begin{align*}
e^{-k\lambda^{p+1}} \leq \varepsilon
\quad\Leftrightarrow\quad k \geq \frac{1}{\lambda^{p+1}} \log \frac{1}{\varepsilon}
\end{align*}

\end{proof}

\begin{proposition}\label{propFaithfullyCapturedIfRuleGraphAppendix}
Let $\mathcal{P}$ be a set of closed path rules. Let $\mathcal{H}$ be a rule graph for $\mathcal{P}$ and let $\eta$ be the corresponding {\modelName} model.  It holds that $\eta$ faithfully captures $\mathcal{P}$.
\end{proposition}
\begin{proof}
Proposition \ref{propRuleGraphCompletenessAppendix} already shows that the first condition from the definition of ``faithfully capturing'' is satisfied. To see why the second condition is satisfied, we can use the same argument as in the proof of Proposition \ref{propProbablyNoFalsePositivesAppendix} to end up with the inequality \eqref{eqConditionABXrandomvariables}. However, because we no longer need to consider randomly initialised entity embeddings, we can simply define $A_j^r=1$ for $j\in Z_i$ and initialise all other coordinates as 0. This ensures that $\mathbf{B_r}\mathbf{Z_a^{(m)}}\preceq \mathbf{Z_b^{(m)}}$ cannot be satisfied (even when $k=1$).
\end{proof}

%********************************************************************************
\section{Constructing Rule Graphs}

\begin{proposition}\label{propAcyclicRuleBaseAppendix}
Let $\mathcal{P}$ be a rule base. Assume that we can rank the relations in $\mathcal{R}$ as $r_1,...,r_{|\mathcal{R}|}$, such that for every rule in $\mathcal{P}$ with $r_i$ in the body and $r_j$ in the head, it holds that $i<j$. There exists a rule graph for $\mathcal{P}$.
\end{proposition}
\begin{proof}
Let $\mathcal{P}$ be a rule base which satisfies the conditions of Proposition \ref{propAcyclicRuleBaseAppendix}, and let $r_1,...,r_{|\mathcal{R}|}$ be the corresponding ranking of the relations. We construct a rule graph $\mathcal{H}$ for $\mathcal{P}$ as follows.
%, where all nodes are of the form $n_R$ with $R\subseteq \mathcal{R}$:
\begin{enumerate}
\item We add the node $n_0$. 
\item For each relation $r\in \mathcal{R}$, we add a node $n_r$, and we connect $n_0$ to $n_r$ with an $r$-edge. 
\item For $i$ going from $|\mathcal{R}|$ to 1:
\begin{enumerate}
\item For each rule $r_{j_1}(X_1,X_2)\wedge ... \wedge r_{j_q}(X_q,X_{q+1})\rightarrow r_i(X_1,X_{q+1})$ with $r_i$ in the head and each $r_i$ edge between nodes $n$ and $n'$ in $\mathcal{H}$, we create fresh nodes $n_1,...,n_q$ and add an $r_{j_1}$-link from $n$ to $n_1$, an $r_{j_2}$ link from $n_1$ to $n_2$, ..., an $r_{j_{q}}$-link from $n_q$ to $n'$.
\end{enumerate}
\end{enumerate}
Clearly the process terminates after a finite number of steps, noting that the new edges that are added for a rule $r_{j_1}(X_1,X_2)\wedge ... \wedge r_{j_q}(X_q,X_{q+1})\rightarrow r_i(X_1,X_{q+1})$ cannot be $r_i$-edges, due to the assumption that $\mathcal{P}$ is free from cyclic dependencies. We also trivially have that condition (R1) is satisfied.

To see why (R2) is satisfied, first note that this is clearly the case after the first two steps have been completed. In the third step, when processing a rule $r_{j_1}(X_1,X_2)\wedge ... \wedge r_{j_q}(X_q,X_{q+1})\rightarrow r_i(X_1,X_{q+1})$ and an edge from $n$ to $n'$, the only existing node where an incoming edge is added is $n'$ (where the other edges end in a fresh node). However, by construction, $n'$ can only have incoming $r_j$-edges with $j\geq i$ whereas $j_q<i$ because of the assumption that $\mathcal{P}$ is free from cyclic dependencies. The addition of the $r_{j_{q}}$-link from $n_q$ to $n'$ can thus not cause (R2) to become unsatisfied. It follows that (R2) still holds after the third step of the construction algorithm is finished. 

Finally, the fact that (R3) and (R4) are satisfied straightforwardly follows from the construction.
\end{proof}

\begin{proposition}\label{propNoCFGrulebasesAppendix}
Let $\mathcal{P}$ be a set of closed path rules and suppose that there exists a rule graph $\mathcal{H}$ for $\mathcal{P}$. Let $\mathcal{R}_1$ be the set of relations that appear in the head of some rule in $\mathcal{P}$. It holds that the language $L_r$ is regular for every $r\in\mathcal{R}_1$. 
\end{proposition}
\begin{proof}
We write $\mathcal{R}_1$ for the set of relations that appear in the head of some rule from the considered rule base, and $\mathcal{R}_2=\mathcal{R}\setminus \mathcal{R}_1$ for the remaining relations.

%\begin{proposition}
%Let $\mathcal{P}$ be a set of closed path rules and suppose that there exists a rule graph $\mathcal{H}$ for $\mathcal{P}$. Let $\mathcal{R}_1$ be the set of relations that appear in the head of some rule in $\mathcal{P}$. It holds that the language $L_r$ is regular for every $r\in\mathcal{R}_1$. 
%\end{proposition}
%\begin{proof}
Let $\alpha(r_i)=r_i$ if $r_i\in\mathcal{R}_2$ and $\alpha(r_i)=\overline{r}_i$ otherwise. We clearly have that $\alpha(r_1)...\alpha(r_k) \in L_r$ iff $\mathcal{P}$ entails the following rule:
$$
r_1(X_1,X_2)\wedge ... \wedge r_{k}(X_{k},X_{k+1}) \rightarrow r(X_1,X_{k+1})
$$
Since we have assumed that $\mathcal{P}$ has a rule graph, thanks to conditions (R3) and (R4), we can check whether this rule is valid by checking whether for each edge labelled with $r$ there is a path connecting the same nodes whose \textit{eq}-reduced type is $r_1;...;r_k$. Let $(n_i,n_j)$ be a an edge labelled with $r$. Then, we can construct a finite state machine (FSM) from $\mathcal{H}$ by treating $n_i$ as the start node and $n_j$ as the unique final node and interpreting $\textit{eq}$ edges as $\varepsilon$-transitions (i.e.\ corresponding to the empty string). Clearly, this FSM will accept the string $r_1...r_k$ if there is a path labelled with $r_1;...;r_k$ connecting $n_i$ to $n_j$. For each edge labelled with $r$, we can construct such an FSM. Let $F_1,...,F_m$ be the languages associated with these FSMs. By construction, $L_r$ is the intersection of $F_1,...,F_m$.  Since $F_1,...,F_m$ are regular, it follows that $L_r$ is regular as well.
\end{proof}

\subsection{Left Regular Rule Bases}
Given a left-regular rule base $\mathcal{P}$, we construct the corresponding rule graph $\mathcal{H}$ as follows.
%, where all nodes are of the form $n_R$ with $R\subseteq \mathcal{R}$:
\begin{enumerate}
\item We add the node $n_0$. 
\item For each relation $r\in \mathcal{R}$, we add a node $n_r$, and we connect $n_0$ to $n_r$ with an $r$-edge. 
\item For each rule of the form $r_1(X,Y)\wedge r_2(Y,Z) \rightarrow r_3(X,Z)$, we add an $r_2$-edge from $n_{r_1}$ to $n_{r_3}$.
\item For each node $n$ with multiple incoming $r$-edges for some $r\in\mathcal{R}$, we do the following. Let $\sharp(r,n)$ be the number of incoming $r$-edges for node $n$. Let $p = \max_{r\in\mathcal{R}}\sharp(r,n)$. We create fresh nodes $n_1,...,n_{p-1}$ and add \textit{eq}-edges from $n_i$ to $n_{i-1}$ ($i\in\{1,...,p-1\}$), where we define $n_0=n$. Let $r\in \mathcal{R}$ be such that $\sharp(r,n)>1$. Let $n'_0,...,n'_q$ be the nodes with an $r$-link to $n$; then we have $q\leq p-1$. For each $i\in\{1,...,q\}$ we replace the edge from $n'_i$ to $n$ by an edge from $n'_i$ to $n_i$.
\end{enumerate}

\noindent We illustrate the construction process with two examples.

\begin{example}\label{exRuleGraphConstructionLoops}
Let $\mathcal{P}$ consist of the following rule:
$$
r_1(X,Y)\wedge r_2(Y,Z) \rightarrow r_1(X,Z)
$$
The corresponding rule graph is shown in Figure \ref{figRuleGraphConstructionLoops}.

\begin{figure}[t]
\centering
\begin{tikzpicture}[scale=0.9]
\begin{scope}[every node/.style={thick,draw}]
    \node (A) at (0,0) {$n_0$};
    \node (B) at (2,1) {$n_{r_1}$};
    \node (C) at (2,-1) {$n_{r_2}$};
    \node (D) at (-2,0) {$n_{\textit{eq}}$};
\end{scope}

\begin{scope}[>={Stealth[black]},
              %every node/.style={fill=white,circle},
              every edge/.style={draw=black,thick}]   
            \path [->] (A) edge node[above] {$r_1$} (B);   
            \path [->] (A) edge node[below] {$r_2$} (C); 
            \path [->] (B) edge[loop right] node[right] {$r_2$} (B); 
            \path [->] (A) edge node[above] {\textit{eq}} (D);    
\end{scope}
\end{tikzpicture}
\caption{Rule graph for Example \ref{exRuleGraphConstructionLoops}. \label{figRuleGraphConstructionLoops}}
\end{figure}
\end{example}

\begin{example}\label{exRuleGraph4}
Let $\mathcal{P}$ contain the following rules: 
\begin{align*}
r_1(X,Y)\wedge r_2(Y,Z) &\rightarrow r_3(X,Z)\\
r_4(X,Y)\wedge r_2(Y,Z) &\rightarrow r_3(X,Z)\\
r_5(X,Y)\wedge r_2(Y,Z) &\rightarrow r_3(X,Z)
\end{align*}
The corresponding rule graph
%, obtained using the process outlined above when starting with the $r_1$-edge from $n_1$ to $n_2$, 
is depicted in Figure \ref{figRuleGraph4}. The nodes $n_1$ and $n_2$ were introduced in step 4 of the construction process. Before this step, there were $r_2$-edges from $n_{r_4}$ to $n_{r_3}$ and from $n_{r_5}$ to $n_{r_3}$. The node $n_{r_3}$ thus had three incoming $r_2$-edges, which violates condition (R2). This is addressed through the use of \textit{eq} edges in step 4.

\begin{figure}[t]
\centering
\begin{tikzpicture}[scale=0.9]
\begin{scope}[every node/.style={thick,draw}]
    \node (A) at (0,0) {$n_0$};
    \node (B) at (-2,0) {$n_{\textit{eq}}$};
    \node (C) at (-2,2) {$n_{r_2}$};
    \node (D) at (0,2) {$n_{r_1}$};     
    \node (E) at (2,2) {$n_{r_3}$};   
    \node (F) at (2,0) {$n_{r_4}$};   
    \node (G) at (4,2) {$n_1$}; 
    \node (H) at (2,-2) {$n_{r_5}$}; 
    \node (I) at (4,0) {$n_2$}; 
\end{scope}

\begin{scope}[>={Stealth[black]},
              %every node/.style={fill=white,circle},
              every edge/.style={draw=black,thick}]   
            \path [->] (A) edge node[below] {\textit{eq}} (B);   
            \path [->] (A) edge node[below] {$r_2$} (C);  
            \path [->] (A) edge node[left] {$r_1$} (D);   
            \path [->] (A) edge node[left] {$r_3$} (E);   
            \path [->] (A) edge node[above] {$r_4$} (F);
            \path [->] (A) edge node[right] {$r_5$} (H);
            \path [->] (D) edge node[above] {$r_2$} (E); 
            \path [->] (F) edge node[left] {$r_2$} (G); 
            \path [->] (G) edge node[above] {\textit{eq}} (E);
            \path [->] (H) edge node[left] {$r_2$} (I);
            \path [->] (I) edge node[right] {\textit{eq}} (G); 
\end{scope}
\end{tikzpicture}
\caption{Rule graph for Example \ref{exRuleGraph4}. \label{figRuleGraph4}}
\end{figure}
\end{example}

%

%The following two propositions show the correctness of the proposed rule graph construction method.
%\begin{proposition}
%The proposed rule graph construction method terminates after a finite number of steps.
%\end{proposition}

\noindent The proposed construction process clearly terminates after a finite number of steps. We show that the proposed construction yields a valid rule graph for $\mathcal{P}$, i.e.\ that the resulting rule graph $\mathcal{H}$ satisfies (R1)--(R4). The fact that (R1) is satisfied follows from the following lemma.

\begin{lemma}\label{lemmaConstructionRuleGraphR1Satisfied}
Let $\mathcal{P}$ be a left-regular set of closed path rules and let $\mathcal{H}$ be the graph obtained using the proposed construction method. For every $r\in\mathcal{R}$, it holds that $\mathcal{H}$ contains an outgoing $r$-edge from $n_0$.
\end{lemma}
\begin{proof}
Let $r\in\mathcal{R}$. The edge from $n_0$ to $n_{r}$ is added in step 2 of the construction process. This edge may be removed in step 4, but in that case, a new $r$-edge is added from $n_0$ to a fresh node.
\end{proof}

\noindent The fact that (R2) is satisfied follows immediately from the construction in step 4. We now move to condition (R3).

%\begin{lemma}\label{lemmaEdgesToXnodesStep1}
%Let $\mathcal{P}$ be a left-regular set of closed path rules and let $\mathcal{H}$ be the graph obtained using the proposed construction method. If $\mathcal{H}$ contains a node $n_R$ with $|R|\geq 2$ then for each $r\in R$ it holds that $r \in\mathcal{R}_1$.
%\end{lemma}
%\begin{proof}
%After step 3 the stated assertion is trivially satisfied as there are no nodes $n_R$ with $|R|\geq 2$. We show that the result remains satisfied after each iteration of step 4. Let $n_{R_1},...,n_{R_p}$ and $n$ be the nodes considered in an iteration of Step 4. If $|R_i|\geq 2$ then the result holds by induction. If $R_i = \{r_i\}$, then we can only have an $r$-edge between $n_{R_i}$ and $n$ if $\mathcal{P}$ has a rule of the form $r_1(X,Y)\wedge r(Y,Z)\rightarrow r_2(X,Z)$, which implies

%\todo{}
%\end{proof}

%\begin{lemma}\label{lemmaEdgesToXnodesStep2}
%Let $\mathcal{P}$ be a left-regular set of closed path rules and let $\mathcal{H}$ be the graph obtained using the proposed construction method. If $\mathcal{H}$ contains an $r$-edge from $n_{R_1}$ to $n_{R_2}$ with $R_1\neq\emptyset$ it holds that $r\in\mathcal{R}_2$.
%\end{lemma}
%\begin{proof}
%\todo{}
%\end{proof}

\begin{lemma}\label{lemmaRulePathsInH}
Let $\mathcal{P}$ be a left-regular set of closed path rules and let $\mathcal{H}$ be the graph obtained using the proposed construction method.  If $\mathcal{P}$ contains the rule $r_1(X_1,X_2) \wedge r_2(X_2,X_3) \rightarrow r_3(X_1,X_3)$, then whenever two nodes $n$ and $n'$ are connected  in $\mathcal{H}$  by a path whose \textit{eq}-reduced type is $r_3$, there is some node $n''$ such that $n$ and $n''$ are connected by a path whose \textit{eq}-reduced type is $r_1$ and $n''$ and $n'$ are connected by a path whose \textit{eq}-reduced type is $r_2$.
\end{lemma}
\begin{proof}
The stated assertion clearly holds after step 3 of the construction method. Indeed, the only $r_3$-edge in $\mathcal{H}$ is from $n_0$ to $n_{r_3}$. Note in particular that no $r_3$ edges can be added in step 3, given our assumption that $\mathcal{P}$ is left-regular. Finally, it is also easy to see that this property remains satisfied after step 4. 
\end{proof}

\noindent The next lemma shows that (R3) is satisfied.

\begin{lemma}\label{lemmaR3Satisfied}
Let $\mathcal{P}$ be a left-regular set of closed path rules and let $\mathcal{H}$ be the graph obtained using the proposed construction method.   Suppose nodes $n$ and $n'$ are connected with an edge of type $r$ and suppose $\mathcal{P}\models r_1(X_1,X_2) \wedge r_2(X_2,X_3) \wedge ... \wedge r_p(X_p,X_{p+1}) \rightarrow r(X_1,X_{p+1})$. Then there is a path whose \textit{eq}-reduced type is $r_1;...;r_p$ from $n$ to $n'$. 
\end{lemma}
\begin{proof}
Assume $\mathcal{P}\models r_1(X_1,X_2) \wedge r_2(X_2,X_3) \wedge ... \wedge r_p(X_p,X_{p+1}) \rightarrow r(X_1,X_{p+1})$. %we know that $r(X_1,X_{p+1})$ can be predicted from $\{r_1(X_1,X_2) \wedge r_2(X_2,X_3) \wedge ... \wedge r_k(X_p,X_{p+1})\}$ using forward chaining. 
Let $n$ and $n'$ be nodes connected by an edge of type $r$. We show the result by structural induction.
First, suppose $p=2$. In this case, the considered rule is of the form $r_1(X_1,X_2) \wedge r_2(X_2,X_3) \rightarrow r(X_1,X_3)$. It then follows from Lemma \ref{lemmaRulePathsInH} that there is a path whose \textit{eq}-reduced type is $r_1;r_2$ connecting $n$ and $n'$.
Let us now consider the inductive case. If $p>3$ then $r_1(X_1,X_2) \wedge r_2(X_2,X_3) \wedge ... \wedge r_p(X_p,X_{p+1}) \rightarrow r(X_1,X_{p+1})$ is derived from at least two rules in $\mathcal{P}$ (given that the rules in $\mathcal{P}$ were restricted to have only two atoms in the body). The last step of the derivation of this rule is done by secting some rule $s_1(X,Y)\wedge s_2(Y,Z)\rightarrow r(X,Z)$ from $\mathcal{P}$ such that
\begin{align*}
\mathcal{P}\models r_1(X_1,X_2) \wedge ... \wedge r_{i-1}(X_{i-1},X_{i}) &\rightarrow s_1(X_1,X_{i})\\
\mathcal{P}\models r_{i}(X_i,X_{i+1}) \wedge ... \wedge r_p(X_p,X_{p+1}) &\rightarrow s_2(X_i,X_{p+1})
\end{align*}
If there is a path from $n$ to $n'$ whose \textit{eq}-reduced type is $r$, we know from Lemma \ref{lemmaRulePathsInH} that there must be a path from  $n$ to $n''$ with \textit{eq}-reduced type $s_1$-edge and a path from $n''$ to $n'$ with \textit{eq}-reduced type $s_2$, for some node $n''$ in $\mathcal{H}$. By induction, we furthermore know that there must then be a path with \textit{eq}-reduced type $r_1;...;r_{i-1}$ from $n$ to $n''$ and a path with \textit{eq}-reduced type $r_i;...;r_{p}$ from $n''$ to $n'$. Thus, we find that there must be a path with \textit{eq}-reduced type $r_1;...;r_p$  from $n$ to $n'$.
\end{proof}

The fact that (R4) is satisfied follows from the next lemma. 

\begin{lemma}\label{lemmaRuleGraphR4Satisfied}
Let $\mathcal{P}$ be a left-regular set of closed path rules and let $\mathcal{H}$ be the graph obtained using the proposed construction method.
Suppose there is a path in $\mathcal{H}$ from $n_0$ to $n_{r}$ whose \textit{eq}-reduced type is $r_1;...;r_p$. Then it holds that $\mathcal{P}\models r_1(X_1,X_2)\wedge ... \wedge r_p(X_p,X_{p_1})\rightarrow r(X_1,X_{p+1})$.
\end{lemma}
\begin{proof}
The result clearly holds after step 2. We show that the result remains valid after each iteration of step 3. Suppose in step 3 we add an $r_2$-edge between $n_{r_1}$ and $n_{r_3}$. This means that:
\begin{align*}
\mathcal{P} \models r_1(X,Y)\wedge r_2(Y,Z)\rightarrow r_3(X,X)
\end{align*}
Let $\tau$ be a path from $n_0$ to $n_r$. If $\tau$ does not contain the new $r_2$-edge, then the fact that the result is valid for $\tau$  follows by induction. Now, suppose that $\tau$ contains the new $r_2$ edge. Then $\tau$ is of the form $r_{i_1};...;r_{i_s};r_2;r_{j_1};...;r_{j_{t}}$. By induction we have:
\begin{align*}
\mathcal{P}&\models r_{i_1}(X_1,X_2)\wedge ... \wedge r_{i_s}(X_s,X_{s+1})\rightarrow r_1(X_1,X_{s+1})
\end{align*}
Clearly there is a path from $n_0$ to $n_{r_3}$ with \textit{eq}-reduced type $r_3$. In particular, there is a path from $n_0$ to $n_{r_3}$ with \textit{eq}-reduced type $r_3;r_{j_1};...r_{j_t}$. By induction, we thus have:
\begin{align*}
\mathcal{P}&\models r_3(X_0,X_1) \wedge r_{j_1}(X_1,X_2)\wedge ... \\
&\quad\quad\quad\wedge r_{j_t}(X_t,X_{t_1})\rightarrow r(X_0,X_{t+1})
\end{align*}
Together we find that the stated result is satisfied.

Finally, we need to show that the result remains satisfied after step 4. This is clearly the case, as this step replaces edges of type $r$ with paths of type $r;\textit{eq};...;\textit{eq}$. The \textit{eq}-reduced types of the paths from $n_0$ to $n_r$ thus remain unchanged after this step.
\end{proof}

\begin{proposition}\label{propLeftRegularRuleGraphAppendix}
For any left-regular set of closed path rules $\mathcal{P}$, there exists a rule graph for $\mathcal{P}$.
\end{proposition}
\begin{proof}
We need to show that the graph $\mathcal{H}$ obtained using the proposed construction process satisfies (R1)--(R4). The fact that (R1), (R3) and (R4) are satisfied follows immediately from Lemmas \ref{lemmaConstructionRuleGraphR1Satisfied}, \ref{lemmaR3Satisfied} and \ref{lemmaRuleGraphR4Satisfied}. The fact that (R2) is satisfied follows trivially from the construction.
\end{proof}

%***********************************
\subsection{Bounded Inference}

Let $\textit{paths}^m_{\mathcal{G}}(b)$ be the set of all paths in $\mathcal{G}$ of length at most $m$ which are ending in $b$.

\begin{lemma}\label{lemmaGNNconvergenceFinite}
For any entity $e\in\mathcal{E}$ it holds that
\begin{align*}
\mathbf{e^{(m)}}\preceq\max\Big(\mathbf{e^{(0)}}, \max_{\pi\in \textit{paths}^m_{\mathcal{G}}(e)} \mu_{\mathit{rels}(\pi)}\big(\textit{emb}_0(\textit{head}(\pi))\big)\Big)
\end{align*}
\end{lemma}
\begin{proof}
This follows immediately from the construction of the GNN.
\end{proof}

\begin{lemma}\label{lemmaNotEntailedZemptyFinite}
Let $\ell$ be the number of nodes in the given $m$-bounded rule graph.
Suppose $\mathcal{P}\cup\mathcal{G}\not \models_m (a,r,b)$. Then there is some $i\in\{1,...,\ell\}$ such that:
\begin{itemize}
\item $Z_i\subseteq I_r$; and
\item whenever $\pi\in \textit{paths}^{m+1}_{\mathcal{G}}(b)$ with $\textit{head}(\pi)=a$, it holds that
$
I_{\textit{rels}(\pi)} \cap Z_i = \emptyset
$.
\end{itemize}
\end{lemma}
\begin{proof}
This lemma is shown in exactly the same way as Lemma \ref{lemmaNotEntailedZempty}, simply replacing $\textit{paths}_{\mathcal{G}}(b)$ by $\textit{paths}^{m+1}_{\mathcal{G}}(b)$ and replacing Condition (R4) by Condition (R4m).
\end{proof}

\begin{proposition}\label{propProbablyNoFalsePositivesBoundedAppendix}
Let $\mathcal{P}$ be a set of closed path rules $\mathcal{G}$ an \textit{eq}-complete knowledge graph.
Let $\mathcal{H}$ be an $m$-bounded rule graph for $\mathcal{P}$ and let $\mathbf{Z_e^{(l)}}$ be the entity representations that are learned by the GNN,  for the corresponding {\modelName} model.  For any $\varepsilon>0$, there exists some $k_0\in \mathbb{N}$ such that, when $k\geq k_0$, for any $i\leq m+1$ and $(a,r,b)\in\mathcal{E}\times\mathcal{R}\times\mathcal{E}$ such that $\mathcal{P}\cup\mathcal{G}\not\models_m (a,r,b)$, we have
\begin{align*}
\textit{Pr}[ \mathbf{B_r}\mathbf{Z_a^{(i)}}\preceq \mathbf{Z_b^{(i)}}] \leq \varepsilon
\end{align*}
\end{proposition}
\begin{proof}
This result is shown in the same way as Proposition \ref{propProbablyNoFalsePositivesAppendix}, by relying on Lemma \ref{lemmaNotEntailedZemptyFinite} instead of Lemma \ref{lemmaNotEntailedZempty}.
\end{proof}

Given a set of closed path rules $\mathcal{P}$ we can construct an $m$-bounded rule graph as follows.
\begin{enumerate}
\item We add the node $n_0$. 
\item For each relation $r\in \mathcal{R}$, we add a node $n_r$, and we connect $n_0$ to $n_r$ with an $r$-edge. 
\item We repeat the following until convergence. Let  $r\in\mathcal{R}$ and assume there is an $r$-edge from $n$ to $n'$. Let $r_1(X,Y)\wedge r_2(Y,Z)\rightarrow r(X,Z)$ be a rule from $\mathcal{P}$ and suppose that there is no $r_1;r_2$ path connecting $n$ and $n'$. Suppose furthermore that the edge $(n,n')$ is on some path from $n_0$ to a node $n_{r'}$, with $r'\in \mathcal{R}$, whose length is at most $m$. We add a fresh node $n''$ to the rule graph, an $r_1$-edge from $n$ to $n''$, and an $r_2$-edge from $n''$ to $n'$.
\item For each $r\in\mathcal{R}$ and $r$-edge $(n,n')$ such that for some rule $r_1(X,Y)\wedge r_2(Y,Z)\rightarrow r(X,Z)$ from $\mathcal{P}$ there is no $r_1;r_2$ path connecting $n$ and $n'$, we do the following:
\begin{enumerate}
\item We add a fresh node $n''$, an $r_1$-edge from $n$ to $n''$ and an $r_2$-edge from $n''$ to $n'$.
\item We repeat the following until convergence. For each $r'$-edge from $n$ to $n''$ and each rule $r_1'(X,Y)\wedge r_2'(Y,Z)\rightarrow r'(X,Z)$ from $\mathcal{P}$, we add an $r_1'$ edge from $n$ to $n''$ and an $r'_2$-loop to $n''$ (if no such edges/loops exist yet).
\item We repeat the following until convergence. For each $r'$-edge from $n''$ to $n'$ and each rule $r_1'(X,Y)\wedge r_2'(Y,Z)\rightarrow r'(X,Z)$ from $\mathcal{P}$, we add an $r_1'$-loop to $n''$ and an $r_2'$-edge from $n''$ to $n'$ (if no such edges/loops exist yet).
\item We repeat the following until convergence. For each $r'$-loop at $n''$, and each rule $r_1'(X,Y)\wedge r_2'(Y,Z)\rightarrow r'(X,Z)$ from $\mathcal{P}$, we add an $r_1'$-loop and an $r_2'$-loop to $n''$ (if no such loops exist yet).
\end{enumerate}
\item For each node $n$ with multiple incoming $r$-edges for one or more relations from $\mathcal{R}$, we do the following. Let $\sharp(r,n)$ be the number of incoming $r$-edges for node $n$. Let $p = \max_{r\in\mathcal{R}}\sharp(r,n)$. We create fresh nodes $n_1,...,n_{p-1}$ and add \textit{eq}-edges from $n_i$ to $n_{i-1}$ ($i\in\{1,...,p-1\}$), where we define $n_0=n$. Let $r\in \mathcal{R}$ be such that $\sharp(r,n)>1$. Let $n'_0,...,n'_q$ be the nodes with an $r$-link to $n$; then we have $q\leq p-1$. For each $i\in\{1,...,q\}$ we replace the edge from $n'_i$ to $n$ by an edge from $n'_i$ to $n_i$.
\end{enumerate}

We provide two examples to illustrate the construction process.

\begin{example}\label{exBoundedRuleGraphConstruction1}
Let us consider the following set of rules:
\begin{align*}
r_1(X,Y) \wedge r_2(Y,Z) & \rightarrow r_3(X,Z)\\
r_3(X,Y) \wedge r_1(Y,Z) & \rightarrow r_2(X,Z)
\end{align*}
The corresponding $1$-bounded rule graph is shown in Fig.\ \ref{figBoundedRuleGraphConstruction1}.

\begin{figure}
\centering
\begin{tikzpicture}[scale=0.9]
\begin{scope}[every node/.style={thick,draw}]
    \node (A) at (0,0) {$n_0$};
    \node (B) at (-1,2) {$n_{r_1}$};
    \node (C) at (1,2) {$n_{\textit{eq}}$};    
    \node (D) at (-2,1) {$n_1$};     
    \node (E) at (-4,0) {$n_2$};   
    \node (F) at (-2,-1) {$n_{r_3}$};   
    \node (G) at (2,1) {$n_3$}; 
    \node (H) at (4,0) {$n_4$}; 
    \node (I) at (2,-1) {$n_{r_2}$}; 
\end{scope}

\begin{scope}[>={Stealth[black]},
              %every node/.style={fill=white,circle},
              every edge/.style={draw=black,thick}]   
            \path [->] (A) edge node[left] {$r_1$} (B);   
            \path [->] (A) edge node[left] {\textit{eq}} (C);  
            \path [->] (A) edge node[below] {$r_1$} (D);   
            \path [->] (D) edge node[above] {$r_1,r_3$ \phantom{ab}} (E);   
            \path [->] (D) edge node[left] {$r_2$} (F);
            \path [->] (E) edge node[above] {$r_1$} (F);
            \path [->] (A) edge node[above] {$r_3$} (F);
            \path [->] (A) edge node[above] {$r_2$} (I); 
            \path [->] (A) edge node[above] {$r_1$} (G); 
            \path [->] (A) edge node[above] {$r_3$} (H); 
            \path [->] (G) edge node[above] {\phantom{ab} $r_2,r_1$} (H);
            \path [->] (H) edge node[above] {$r_1$} (I);       
            \path [->] (E) edge[loop below] node[below] {$r_1,r_2,r_3$} (E);             
            \path [->] (G) edge[loop above] node[above right] {$r_1,r_2,r_3$} (G); 
\end{scope}
\end{tikzpicture}
\caption{Rule graph for Example \ref{exBoundedRuleGraphConstruction1}. \label{figBoundedRuleGraphConstruction1}}
\end{figure}
\end{example}

\begin{example}\label{exBoundedRuleGraphConstruction2}
Let us consider the following set of rules:
\begin{align*}
r_1(X,Y) \wedge r_2(Y,Z) &\rightarrow r_3(X,Z)\\
r_4(X,Y) \wedge r_5(Y,Z) &\rightarrow r_1(X,Z)\\
r_4(X,Y) \wedge r_5(Y,Z) &\rightarrow r_2(X,Z)
\end{align*}
The corresponding $2$-bounded rule graph is shown in Fig.\ \ref{figBoundedRuleGraphConstruction2}. Note how this graph is in fact also a rule graph: due to the fact that there are no cyclic dependencies in the rule base $\mathcal{P}\cup\mathcal{G}\models_2 (e,r,g)$ is equivalent with   $\mathcal{P}\cup\mathcal{G}\models (e,r,g)$.

\begin{figure}
\centering
\begin{tikzpicture}[scale=0.9]
\begin{scope}[every node/.style={thick,draw}]
    \node (A) at (0,0) {$n_0$};
    \node (B) at (0,2) {$n_1$};
    \node (C) at (2,1) {$n_2$};    
    \node (D) at (4,2) {$n_3$};     
    \node (E) at (4,0) {$n_{r_3}$};   
    \node (F) at (3,-1) {$n_4$};   
    \node (G) at (1,-2) {$n_{r_2}$}; 
    \node (H) at (-1,-2) {$n_{r_1}$}; 
    \node (I) at (-3,-1) {$n_5$}; 
    \node (J) at (-4,0) {$n_{r_4}$}; 
    \node (K) at (-3,1) {$n_{r_5}$}; 
    \node (L) at (-2,2) {\textit{eq}}; 
\end{scope}

\begin{scope}[>={Stealth[black]},
              %every node/.style={fill=white,circle},
              every edge/.style={draw=black,thick}]   
            \path [->] (A) edge node[above] {$r_4$} (J); 
            \path [->] (A) edge node[above] {$r_5$} (K);  
            \path [->] (A) edge node[right] {\textit{eq}} (L); 
            \path [->] (A) edge node[right] {$r_4$} (B);   
            \path [->] (B) edge node[above] {$r_5$} (C);   
            \path [->] (A) edge node[above] {$r_1$} (C);   
            \path [->] (C) edge node[above] {$r_4$} (D);   
            \path [->] (D) edge node[left] {$r_5$} (E);   
            \path [->] (C) edge node[above] {$r_2$} (E);   
            \path [->] (A) edge node[above] {$r_3$} (E);   
            \path [->] (A) edge node[below] {$r_4$} (F);               
            \path [->] (F) edge node[above] {$r_5$} (G);   
            \path [->] (A) edge node[right] {$r_2$} (G);   
            \path [->] (A) edge node[left] {$r_1$} (H);   
            \path [->] (A) edge node[below] {$r_4$} (I);   
            \path [->] (I) edge node[above] {$r_5$} (H); 
\end{scope}
\end{tikzpicture}
\caption{Rule graph for Example \ref{exBoundedRuleGraphConstruction2}. \label{figBoundedRuleGraphConstruction2}}
\end{figure}
\end{example}

The construction process clearly terminates after a finite number of steps. Indeed, only edges that are on a path of length $m$ are expanded in step 3, and given that there are only finitely many such paths, step 3 must terminate. It is also straightforward to see that the other steps must terminate. We now show that the construction process yields a valid $m$-bounded rule graph.

Conditions (R1) and (R2) are clearly satisfied.  Next, we show that condition (R3) is satisfied.
\begin{lemma}\label{lemmaABC0}
Let $\mathcal{P}$ be a set of closed path rules and let $\mathcal{H}$ be the resulting $m$-bounded rule graph, constructed using the proposed process.  Suppose nodes $n$ and $n'$ are connected with an edge of type $r$ and suppose $\mathcal{P}\models r_{i_1}(X_1,X_2) \wedge r_{i_2}(X_2,X_3) \wedge ... \wedge r_{i_p}(X_p,X_{p+1}) \rightarrow r(X_1,X_{p+1})$. Then there is a path connecting $n$ to $n'$, whose \textit{eq}-reduced type is $r_{i_1};...;r_{i_p}$. 
\end{lemma}
\begin{proof}
First, we show that at the end of step 4, there must be a path of type $r_{i_1};...;r_{i_p}$ connecting $n$ and $n'$.
By construction, we immediately have that whenever two nodes $(n,n')$ are connected with an $r_i$-edge and $\mathcal{P}$ contains the rule $r_j(X,Y)\wedge r_l(Y,Z)\rightarrow r_i(X,Z)$ it holds that there exists some node $n''$ such that there is an $r_j$-edge from $n$ to $n''$ and an $r_l$ edge from $n''$ to $n'$. The existence of a path of type $r_{i_1};...;r_{i_p}$ then follows in the same way as in the proof of Lemma \ref{lemmaR3Satisfied}. It remains to be shown that the proposition remains valid after step 5. However, the paths in the final graph are those that can be found in the graph after step 4, with the possible addition of some \textit{eq}-edges. This means in particular that after step 5, there must still be a path from $n$ to $n'$ whose \textit{eq}-reduced type is $r_{i_1};...;r_{i_p}$.
\end{proof}

Finally, the fact that (R4m) is satisfied follows from the following lemma.

\begin{lemma}\label{lemmaABC1}
Let $\mathcal{P}$ be a set of closed path rules, and let $\mathcal{H}$ be the resulting $m$-bounded rule graph, constructed using the process outlined above. Suppose there is a path from $n_0$ to $n_r$ whose \textit{eq}-reduced type if $r_1;...;r_p$, with $p\leq m+1$. Then it holds that $\mathcal{P} \models r_1(X_1,X_2)\wedge ... \wedge r_p(X_p,X_{p_1})\rightarrow r(X_1,X_{p+1})$.
\end{lemma}
\begin{proof}
We clearly have that the proposition holds after step 3 of the construction method. After step 3, if there is an $r$-link between nodes $n$ and $n'$ and a rule $r_1(X,Y)\wedge r_2(Y,Z)\rightarrow r(X,Z)$ such that $n$ and $n'$ are not connected by an $r_1;r_2$ path, it must be the case that any path from $n_0$ to some node $n_r$ which contains the edge $(n,n')$ must have a length of at least $m+1$. It follows that any path from $n_0$ to some node $n_r$ which contains an edge that was added during step 4 must have length at least $m+2$. We thus have in particular that the proposition still holds after step 4. The paths in the final graph are those that can be found in the graph after step 4, with the possible addition of some \textit{eq}-edges. Since the proposition only depends on the \textit{eq}-reduced types of the paths, the result still holds after step 5.
\end{proof}

Together, we have shown the following result.
\begin{proposition}\label{propAlwaysMboundedRuleGraphAppendix}
For any set of closed path rules $\mathcal{P}$, there exists an $m$-bounded rule graph for $\mathcal{P}$.
\end{proposition}
\begin{proof}
Let $\mathcal{P}$ be a set of closed path rules and let $\mathcal{H}$ be the graph obtained using the proposed construction method for $m$-bounded rule graphs. We need to show that $\mathcal{H}$ satisfies (R1)--(R3) and (R4m). (R1) and (R2) are trivially satisfied. The fact that (R3) is satisfied was shown in Lemma \ref{lemmaABC0}, while the fact that (R4m) is satisfied was shown in Lemma \ref{lemmaABC1}.
\end{proof}

%****************************************************
\section{Beyond Closed Path rules}
We start by considering a graph-based abstraction of the model, similar to the notion of rule graph that we used for encoding closed path rules. To avoid confusion, we will refer to the graphs constructed here as \emph{intersection graphs}. Let $\mathcal{P}$ be an arbitrary set of intersection and hierarchy rules, defined over a set of relations $\mathcal{R}$. We construct the corresponding intersection graph as follows:
\begin{enumerate}
\item We add a node $n_0$.
\item For each relation $r\in\mathcal{R}$, we add a node $n_r$, and we connect $n_0$ to $n_r$ with an $r$-edge.
\item We repeat the following steps until convergence
\begin{enumerate}
\item For each hierarchy rule of the form $r_1(X,Y)\rightarrow r_2(X,Y)$ in $\mathcal{P}$, and every $r_2$-edge from $n_0$ to some node $n$, we add an $r_1$-edge from $n_0$ to $n$, if such an edge does not already exist.
\item For each intersection rule of the form $r_1(X,Y) \wedge r_2(X,Y) \rightarrow r_3(X,Y)$ in $\mathcal{P}$, and every $r_3$-edge from $n_0$ to some node $n$, such that there is no $r_1$-edge and no $r_2$-edge from $n_0$ to $n$, we complete the following steps:
\begin{enumerate}
\item We add a new node $n'$ to the graph.
\item For each edge from $n_0$ to $n$ labelled with some relation $r$, we add an $r$-edge from $n_0$ to $n'$.
\item We add an $r_1$-edge from $n_0$ to $n$, and an $r_2$-edge from $n_0$ to $n'$.
\end{enumerate}
\end{enumerate}
\end{enumerate}
Note that the construction process must terminate after a finite number of steps, given that there are only finitely many edges that can be added between two nodes, and each time a node is duplicated in step 3(b), the number of edges between $n_0$ and that node is increased.
Let $\mathcal{H}$ be the resulting graph. We construct a {\modelName} model from $\mathcal{H}$ in the same way as for rule graphs. In particular, the matrices $\mathbf{B_r}=(b_{ij})$ are again defined as in \eqref{eqDefGNNfromRuleGraph}.
Note that by construction, Condition (R1) is clearly satisfied for $\mathcal{H}$. Moreover, the only edges in $\mathcal{H}$ are between $n_0$ and the other nodes, which means that Condition (R2) is also satisfied. It follows that the constructed {\modelName} model is indeed valid, i.e.\ that the resulting $\mathbf{B_r}$ matrices satisfy the constraint that each row is either a one-hot vector or a 0-vector.

\begin{lemma}\label{lemmaIntersectionGraphIntersections}
Assume that $\mathcal{P}\models r_1(X,Y)\wedge ... \wedge r_p(X,Y)\rightarrow r(X,Y)$. Assume that $\mathcal{H}$ contains an $r$-edge between $n_0$ and some node $n$. Then it holds that $\mathcal{H}$ also contains an edge between $n_0$ and $n$ which is labelled with one of $r_1,...,r_p$.
\end{lemma}
\begin{proof}
First note that the result clearly holds for the rules that are contained in $\mathcal{P}$.  Indeed, if $p=1$ and $\mathcal{P}$ contains the rule $r_1(X,Y) \rightarrow r(X,Y)$, then the result holds because of step 3(a) of the construction process. If $p=2$ and $\mathcal{P}$ contains the rule $r_1(X,Y)\wedge r_2(X,Y)\rightarrow r(X,Y)$, then the result holds because of step 3(b) of the construction process. We now show, by structural induction, that the result remains satisfied for rules that are derived from $\mathcal{P}$.

Suppose that the result holds for the rule $r_1(X,Y)\wedge ... \wedge r_i' (X,Y) \wedge ... \wedge r_p(X,Y)\rightarrow r(X,Y)$ and $\mathcal{P}$ contains the rule $r_i(X,Y)\rightarrow r_i'(X,Y)$. We then have that there is an edge between $n_0$ and $n$ labelled with one of $r_1,...,r_i',...,r_p$. However, if there is an edge labelled with $r_i'$, then because of step 3(a) of the construction process, there will also be an edge labelled with $r_i$. We thus find that there must be an edge labelled with one of $r_1,...,r_i,...,r_p$, meaning that the result holds for the rule $r_1(X,Y)\wedge ... r_i (X,Y) \wedge ... \wedge r_p(X,Y)\rightarrow r(X,Y)$.

Finally, assume that the result holds for $r_1(X,Y)\wedge ... \wedge r_{i-1} (X,Y) \wedge r_i' (X,Y) \wedge r_{i+2} (X,Y) \wedge ... \wedge r_p(X,Y)\rightarrow r(X,Y)$ and $\mathcal{P}$ contains the rule $r_i(X,Y) \wedge r_{i+1}(X,Y) \rightarrow r_i'(X,Y)$.  We then have that there is an edge between $n_0$ and $n$ labelled with one of $r_1,...,r_i',r_{i+2},...,r_p$. However, if there is an edge labelled with $r_i'$, then because of step 3(b) of the construction process, there will also be an edge labelled with either $r_i$ or $r_{i+1}$. We thus find that there must be an edge labelled with one of $r_1,...,r_p$, meaning that the result holds for the rule $r_1(X,Y)\wedge ... r_i (X,Y) \wedge ... \wedge r_p(X,Y)\rightarrow r(X,Y)$.
\end{proof}

\begin{lemma}\label{lemmaPropHierarchyIntersectionA}
Consider a triple $(a,r,b)$ such that $\mathcal{P}\cup\mathcal{G}\models (a,r,b)$. Let $\eta$ be the {\modelName} embedding constructed above, and let $\tau$ be an entity embedding such that $(\tau,\eta)$ captures every triple in $\mathcal{G}$. It holds that $(\tau,\eta)$ captures $(a,r,b)$.
\end{lemma}
\begin{proof}
Let us write $\mathbf{Z_e}$ for the entity embedding $\tau(e)$.
The fact that $\mathcal{P}\cup\mathcal{G}\models (a,r,b)$ holds means that $\mathcal{G}$ either contains the triple $(a,r,b)$ itself, or that $\mathcal{G}$ contains triples of the form $(a,r_1,b),...,(a,r_p,b)$ such that $\mathcal{P}\models r_1(X,Y)\wedge ... \wedge r_p(X,Y) \rightarrow r(X,Y)$. In the former case, the result is trivial 

Now assume that $\mathcal{G}$ contains triples of the form $(a,r_1,b),...,(a,r_p,b)$ such that $\mathcal{P}\models r_1(X,Y)\wedge ... \wedge r_p(X,Y) \rightarrow r(X,Y)$. Then by assumpion, we have that the following conditions are satisfied:
\begin{align*}
\mathbf{B_{r_1}}\mathbf{Z_a}&\preceq \mathbf{Z_b} \\
&...\\
\mathbf{B_{r_p}}\mathbf{Z_a}&\preceq \mathbf{Z_b} 
\end{align*}
We need to show that the following condition is satisfied:
$$
\mathbf{B_r}\mathbf{Z_a}\preceq \mathbf{Z_b} 
$$
Suppose this condition was violated for the component on row $i$ and column $j$. We then find that the following inequality must also be violated (for the same component):
$$
\mathbf{B_r}\mathbf{Z_a}\preceq \max(\mathbf{B_{r_1}}\mathbf{Z_a},...,\mathbf{B_{r_p}}\mathbf{Z_a})
$$
Let us write $x_r$ for the value at position $(i,j)$ in the matrix $\mathbf{B_r}\mathbf{Z_a}$, and similar for $x_{r_1},...,x_{r_p}$. The fact that the latter condition is violated then means:
$$
x_r > \max(x_{r_1},...,x_{r_p})
$$
This is only possible if row $i$ of $\mathbf{B_r}$ is a one-hot vector with the non-zero component in some position $q$ such that $\mathbf{B_{r_1}},...,\mathbf{B_{r_p}}$ are all 0 at position $(i,q)$. By definition of the matrices $\mathbf{B_r}$, i.e.\ \eqref{eqDefGNNfromRuleGraph}, this means that $\mathcal{H}$ has an $r$-edge from $n_q$ to $n_i$, while there are no edges labelled with any of the relations $r_1,...,r_p$  between these two nodes. However, using Lemma \ref{lemmaIntersectionGraphIntersections}, we know that this is not possible, given that we had assumed $\mathcal{P}\models r_1(X,Y)\wedge ... \wedge r_p(X,Y) \rightarrow r(X,Y)$.
\end{proof}

In the following, we will write $\mathbf{Z_a^{(m)}}$ to denote the entity representations that are learned by the GNN that is associated with the constructed {\modelName} model. We will furthermore write $\mathbf{B_r}$ to denote the parameters of the corresponding model.

\begin{lemma}\label{lemmaIntersectionGraphUpdates0}
For each $m\geq 1$, it holds that $\mathbf{B_r}\mathbf{Z_a^{(m)}}=\mathbf{B_r}\mathbf{Z_a^{(0)}}$.
\end{lemma}
\begin{proof}
This follows immediately from the fact that all edges in $\mathcal{H}$ start from $n_0$, together with the fact that $n_0$ does not have any incoming edges. 
\end{proof}
\begin{corollary}\label{corIntersectionGraphAllUpdatesSame}
For each $m\geq 1$, it holds that $\mathbf{Z_a^{(m)}}=\mathbf{Z_a^{(1)}}$.
\end{corollary}

\begin{lemma}\label{lemmaIntersectionsComplete}
Suppose that for every $r$-edge in $\mathcal{H}$, it holds that there is an edge between the same two nodes which is labelled with one of $r_1,...,r_p$. We have that $\mathcal{P}\models r_1(X,Y)\wedge ... \wedge r_p(X,Y) \rightarrow r(X,Y)$.
%Suppose we have $\mathcal{P}\not\models r_1(X,Y)\wedge ... \wedge r_p(X,Y)\rightarrow r(X,Y)$. Then $\mathcal{H}$ contains an $r$-edge from $n_0$ to some node $n$, such that none of the edges between $n_0$ and $n$ are labelled with any of $r_1,...,r_p$.
\end{lemma}
\begin{proof}
The result is trivially satisfied after step 2 of the construction process, as any two nodes are connected by at most one edge at this point. We show, by induction, that the result remains satisfied after each iteration of step 3.  

Let us first consider step 3(a). Assume that prior to executing this step, there are $r$-edges between $n_0$ and each of $n_1,...,n_q$. Let us write $r^*(X,Y)\rightarrow s(X,Y)$ for the hierarchy rule that is considered in this step. Suppose that after executing the step, for each $n_i$, with $i\in \{1,...,q\}$, there is an edge between $n_0$ and $n_i$ which is labelled with one of $r_1,...,r_p$. Prior to executing this step, either the same must have been true for $\{r_1,...,r_p\}$, or this property must have been true for $\{r_1,...,r_p\} \cup \{s\}$ such that $r^*\in \{r_1,...,r_p\}$. In the former case, we obtain  $\mathcal{P}\models r_1(X,Y)\wedge ... \wedge r_p(X,Y) \rightarrow r(X,Y)$ by induction. In the latter case, we must have that $r^*=r_j$ for some $j\in \{1,...,p\}$. By induction we then have $\mathcal{P}\models  r_1(X,Y)\wedge ... \wedge r_p(X,Y) \wedge s(X,Y)\rightarrow r(X,Y)$. Since $\mathcal{P}$ contains the rule $r_j(X,Y)\rightarrow s(X,Y)$, we again find  $\mathcal{P}\models r_1(X,Y)\wedge ... \wedge r_p(X,Y) \rightarrow r(X,Y)$.

Now consider step 3(b), and again assume that prior to executing this step, there are $r$-edges between $n_0$ and each of $n_1,...,n_q$. Let us write $r^*(X,Y) \wedge r'(X,Y) \rightarrow s(X,Y)$ for the intersection rule that is considered in this step. Suppose that after executing the step, for each $n_i$, with $i\in \{1,...,q\}$, there is an edge between $n_0$ and $n_i$ which is labelled with one of $r_1,...,r_p$. Prior to executing this step, either the same must have been true for $\{r_1,...,r_p\}$, or this property must have been true for $\{r_1,...,r_p\} \cup \{s\}$ such that $\{r^*,r'\}\subseteq \{r_1,...,r_p\}$.  Similar as for step 3(a), in both cases we find that  $\mathcal{P}\models r_1(X,Y)\wedge ... \wedge r_p(X,Y) \rightarrow r(X,Y)$.
\end{proof}

% \noindent Recall that the GNN updates are defined as follows.
% \begin{align} \label{eqGNNformB}
% \mathbf{Z_f^{(l+1)}} = \max\big(\{\mathbf{Z_f^{(l)}}\} \cup \{\mathbf{B_r}\mathbf{Z_e^{(l)}}\,|\, (e,r,f)\in \mathcal{G}\}\big)
% \end{align}

\begin{lemma}\label{lemmaPropHierarchyIntersectionB}
For any $\varepsilon>0$ there is a $k_0\in\mathbb{N}$ such that, when $k\geq k_0$, for any $m\in \mathbb{N}$ and $(a,r,b)\in\mathcal{E}\times\mathcal{R}\times\mathcal{E}$ such that $\mathcal{P}\cup\mathcal{G} \not\models (a,r,b)$, we have $\textit{Pr}[ \mathbf{B_r}\mathbf{Z_a^{(m)}}\preceq \mathbf{Z_b^{(m)}}] \leq \varepsilon$.
\end{lemma}
\begin{proof}
Suppose $\mathcal{P}\cup\mathcal{G} \not\models (a,r,b)$.
By Lemma \ref{lemmaIntersectionGraphUpdates0} and Corollary \ref{corIntersectionGraphAllUpdatesSame}, we know that $\mathbf{B_r}\mathbf{Z_a^{(m)}}\preceq \mathbf{Z_b^{(m)}}$ is equivalent with:
$$
\mathbf{B_r}\mathbf{Z_a^{(0)}}\preceq \mathbf{Z_b^{(1)}}
$$
Using \eqref{eqGNNformB}, we can rewrite this as:
\begin{align}\label{eqIntersectionProofCompletenessMaxA}
\mathbf{B_r}\mathbf{Z_a^{(0)}}\preceq \max\big(\{\mathbf{Z_b^{(0)}}\} \cup \{\mathbf{B_s}\mathbf{Z_e^{(0)}}\,|\, (e,s,b)\in \mathcal{G}\}\big)
\end{align}
Let $r_1,...,r_p$ be the set of all relations in $\mathcal{R}$ such that $\mathcal{G}$ contains a triple $(a,r_i,b)$. Since we assumed $\mathcal{P}\cup\mathcal{G} \not\models (a,r,b)$, we must have that $\mathcal{P}\not\models r_1(X,Y)\wedge ... \wedge r_p(X,Y) \rightarrow r(X,Y)$. 
From Lemma \ref{lemmaIntersectionsComplete}, we know that there must then exist an edge in $\mathcal{H}$ from $n_0$ to some node $n$, such that there is no edge labelled with any of $r_1,...,r_p$ from $n_0$ to $n$. Let this be the edge associated with coordinate $(i,j)$ in the matrices $\mathbf{B_s}$. Note that because $n_0\neq n$ (since the construction process never adds self-loops), we must have $i\neq j$. This means that the $i\textsuperscript{th}$ row of $\mathbf{B_r}\mathbf{Z_a^{(0)}}$ corresponds to the $j\textsuperscript{th}$ row of $\mathbf{Z_a^{(0)}}$. For each of the triples $(e,s,b)$ we have that either $e\neq a$ or that $r\in \{r_1,...,r_p\}$. Furthermore, since there is no $r_i$-edge from $n_0$ to $n$, we have that the component $(i,j)$ of $\mathbf{B_{r_i}}$ is 0 for each $i\in \{1,...,p\}$. We thus have that the $i\textsuperscript{th}$ row of $\mathbf{B_s}\mathbf{Z_e^{(0)}}$ does not correspond to the $j\textsuperscript{th}$ row of $\mathbf{Z_a^{(0)}}$, for any of the triples $(e,s,b)$, noting that when $s\notin \{r_1,...,r_p\}$ we must have $e\neq a$. Furthermore, we also have that the $i\textsuperscript{th}$ row of $\mathbf{Z_b^{(0)}}$ cannot correspond to the $j\textsuperscript{th}$ row of $\mathbf{Z_a^{(0)}}$. Note that this is the case even when $a=b$ since $i\neq j$. Let us write $\mathbf{a}$ for the $i\textsuperscript{th}$ row of $\mathbf{B_r}\mathbf{Z_a^{(0)}}$ and $\mathbf{b_1},...,\mathbf{b_u}$ for the $i\textsuperscript{th}$ rows  of the arguments of the maximum in the right-hand side of \eqref{eqIntersectionProofCompletenessMaxA}. Then we have
\begin{align*}
\textit{Pr}[ \mathbf{B_r}\mathbf{Z_a^{(m)}}\preceq \mathbf{Z_b^{(m)}}] 
&\leq \textit{Pr}[ \mathbf{a} \preceq \max(\mathbf{b_1},...,\mathbf{b_u}) ]\\
&\leq (1-\lambda^u)^k\\
&\leq  e^{-k\lambda^u}
\end{align*}
with $\lambda$ defined as in the proof of Proposition \ref{propProbablyNoFalsePositivesAppendix}.
The value of $u$ is upper bounded by $\ell\cdot |\mathcal{E}|$, with $\ell$ the number of nodes in the rule graph.
By choosing $k$ sufficiently large, we can thus make this probability arbitrarily small. In particular:
\begin{align*}
e^{-k\lambda^{u}} \leq \varepsilon
\quad\Leftrightarrow\quad k \geq \frac{1}{\lambda^{u}} \log \frac{1}{\varepsilon}
\end{align*}

\end{proof}

\begin{proposition}\label{propHierarchyIntersectionAppendix}
Let $\mathcal{P}$ be a set of hierarchy and intersection rules. There exists a {\modelName} model such that for every knowledge graph $\mathcal{G}$ the following conditions are satisfied:
\begin{enumerate}
\item Suppose that $\mathcal{P}\cup\mathcal{G} \models (a,r,b)$ and let $\tau$ be an entity embedding such that $(\tau,\eta)$ captures every triple in $\mathcal{G}$. It holds that $(\tau,\eta)$ captures $(a,r,b)$.
\item For any $\varepsilon>0$ there is a $k_0\in\mathbb{N}$ such that, when $k\geq k_0$, for any $m\in \mathbb{N}$ and $(a,r,b)\in\mathcal{E}\times\mathcal{R}\times\mathcal{E}$ such that $\mathcal{P}\cup\mathcal{G} \not\models (a,r,b)$, we have $\textit{Pr}[ \mathbf{B_r}\mathbf{Z_a^{(m)}}\preceq \mathbf{Z_b^{(m)}}] \leq \varepsilon$, where $\mathbf{Z_e^{(m)}}$ are the entity representations that are learned by the GNN \eqref{eqGNNformB}.
\end{enumerate}
\end{proposition}
\begin{proof}
This follows immediately from Lemmas \ref{lemmaPropHierarchyIntersectionA} and \ref{lemmaPropHierarchyIntersectionB}.
\end{proof}
%\section{}

%\appendix

\section{Experimental Details}
\label{app:ExpDetails}

This section lists additional details about our experiment's setup, benchmark datasets, and evaluation protocol. 
Section~\ref{app:ImplementationAndReproducibility} provides some additional implementation details.
The origins and licenses of the standard benchmarks for inductive KGC are discussed in Section \ref{app:BenchData}. Section \ref{secBaselines} provides a justification for why certain baselines have not been considered in our experiments.
Details on \modelName's hyperparameter optimisation are discussed in Section \ref{app:ExperimentSetup}.
Finally, details about the evaluation protocol, together with the complete evaluation results, are provided in Section \ref{app:Metrics}.

\subsection{Implementation Details}
\label{app:ImplementationAndReproducibility}
\modelName\ is trained on an NVIDIA Tesla V100 PCIe 32 GB GPU. %of our internal cluster. In particular, 
We train \modelName\ for up to $1000$ epochs, minimizing the margin ranking loss with %stochastic gradient descent and 
the Adam optimiser \citep{Adam}. If the Hits@10 score on the validation split of $\mathcal{G}_{\textit{Train}}$ does not increase by at least $1\%$ within $100$ epochs, we stop the training early. 

\modelName\ was implemented using the Python library PyKEEN 1.10.1 \citep{PyKeen}. PyKEEN employs the MIT license and 
offers numerous benchmarks for KGC, facilitating the comfortable reuse of \modelName's code for upcoming applications and comparisons.

\begin{table}[t]
\centering
\footnotesize
\begin{tabular}{llllllll}
\toprule
          &    & $\mathcal{R}_{\textit{Train}}$ & $\mathcal{E}_{\textit{Train}}$ & $\mathcal{G}_{\textit{Train}}$ & $\mathcal{R}_{\textit{Test}}$ & $\mathcal{E}_{\textit{Test}}$ & $\mathcal{G}_{\textit{Test}}$  \\
          \midrule
 \multirow{4}{*}{\rotatebox[origin=c]{90}{\textbf{\scriptsize FB15k-237}}} & v1 & 180                   & 1594                  & 5226                  & 142                  & 1093                 & 2404                    \\
          & v2 & 200                   & 2608                  & 12085                 & 172                  & 1660                 & 5092                    \\
          & v3 & 215                   & 3668                  & 22394                 & 183                  & 2501                 & 9137                    \\
          & v4 & 219                   & 4707                  & 33916                 & 200                  & 3051                 & 14554                   \\
          \midrule
 \multirow{4}{*}{\rotatebox[origin=c]{90}{\textbf{\scriptsize WN18RR}}}    & v1 & 9                     & 2746                  & 6678                  & 8                    & 922                  & 1991                     \\
          & v2 & 10                    & 6954                  & 18968                 & 10                   & 2757                 & 4863                     \\
          & v3 & 11                    & 12078                 & 32150                 & 11                   & 5084                 & 7470                     \\
          & v4 & 9                     & 3861                  & 9842                  & 9                    & 7084                 & 15157                    \\
          \midrule
 \multirow{4}{*}{\rotatebox[origin=c]{90}{\textbf{\scriptsize NELL-995}}}  & v1 & 14                    & 3103                  & 5540                  & 14                   & 225                  & 1034                     \\
          & v2 & 88                    & 2564                  & 10109                 & 79                   & 2086                 & 5521                     \\
          & v3 & 142                   & 4647                  & 20117                 & 122                  & 3566                 & 9668                     \\
          & v4 & 76                    & 2092                  & 9289                  & 61                   & 2795                 & 8520                    \\
          \bottomrule
\end{tabular}
\caption{Number of relation, entities, and triples of the train, validation, and test split of the training and testing graph of the inductive benchmarks, split by corresponding benchmark versions v1-4.}
\label{tab:BenchCharacteristics}
\end{table}

\subsection{Benchmarks}
\label{app:BenchData}
Table \ref{tab:BenchCharacteristics} states the entity, relation, and triple counts of the training and test graphs, for each of the considered benchmarks. 

We did not find a license for any of the three inductive benchmarks nor their corresponding transductive supersets. Furthermore, WN18RR is a subset of the WordNet database \citep{DBLP:journals/cacm/Miller95}, which states lexical relations of English words. We also did not find a license for this dataset. FB15k-237 is a subset of FB15k \citep{DBLP:conf/nips/BordesUGWY13}, which is a subset of Freebase \citep{DBLP:conf/acl-cvsc/ToutanovaC15}, a collaborative database that contains general knowledge, such as about celebrities and awards, in English. We did not find a license for FB15k-237 but found that FB15k \citep{DBLP:conf/nips/BordesUGWY13} uses the CC BY 2.5 license. Finally, NELL-995 \citep{DBLP:conf/emnlp/XiongHW17} is a subset of NELL \citep{DBLP:conf/aaai/CarlsonBKSHM10}, a dataset that was extracted from semi-structured and natural-language data on the web and that includes information about e.g., cities, companies, and sports teams. Also for NELL, we did not find any license information.

% \paragraph{Characteristics} Table \ref{tab:BenchCharacteristics} lists properties of the inductive benchmarks (taken from \citet{DBLP:conf/coling/AnilGIS24}), such as the number of triples and entities of the train, validation and test split of the training and testing graph of each dataset version. 

\subsection{Baselines}\label{secBaselines}
As explained in the main text, we train \modelName\ on the train split of $\mathcal{G}_{\textit{Train}}$, tune on the validation split of $\mathcal{G}_{\textit{Train}}$, and evaluate on the test split of $\mathcal{G}_{\textit{Test}}$. As discussed by \citet{DBLP:conf/coling/AnilGIS24}, some approaches in the literature have been evaluated in different ways, e.g.\ by tuning hyperparameters on the validation split of $\mathcal{G}_{\textit{Test}}$, and their reported results are thus not directly comparable. This is the case, among others, for RED-GNN \cite{DBLP:conf/www/ZhangY22}, CBGNN \cite{DBLP:conf/icml/YanMGT022}, Node-
Piece \cite{DBLP:conf/iclr/0001DWH22} and ReFactor GNN \cite{DBLP:conf/nips/ChenM0MS022}. For this reason, these models have not been included in our results table.

\subsection{Hyperparameter Optimisation}
\label{app:ExperimentSetup}

% We train \modelName\ on the train split of $\mathcal{G}_\textit{{Train}}$, tune our model's hyperparameters on the validation split of $\mathcal{G}_{\textit{Train}}$, and finally evaluate the performance of the best model on the test split of $\mathcal{G}_{\textit{Test}}$. \modelName\ is trained on an NVIDIA Tesla V100 PCIe 32 GB GPU of our internal cluster. In particular, we train \modelName\ for up to $1000$ epochs, minimizing the margin ranking loss with stochastic gradient descent and the Adam optimizer \cite{Adam}. If the Hits@10 score on the validation split of $\mathcal{G}_{\textit{Train}}$ does not increase by at least $1\%$ within $100$ epochs, we stop the training early. To account for small performance fluctuations, we repeat our experiments three times and report the average performance. For the final evaluation, we select the hyperparameter configuration with the highest Hits@10 score on the validation split of $\mathcal{G}_{\textit{Train}}$. In accordance with \cite{DBLP:conf/icml/TeruDH20}, we evaluate \modelName's test performance on $50$ negatively sampled entities per triple of the test split of $\mathcal{G}_{\textit{Test}}$ and report the Hits@10 scores (see Section \ref{app:Metrics} for details). In the following, we discuss \modelName's hyperparameters.

% \paragraph{hyperparameter Optimisation} 

\begin{table}[t]
\centering
\footnotesize
\begin{tabular}{lcccccc}
          \toprule
          &    & \#Layers & $l$ & $k$ & $\lambda$ & lr    \\
          \midrule
\multirow{4}{*}{\rotatebox[origin=c]{90}{\textbf{\scriptsize FB15k-237}}}  & v1 & 4        & 25  & 80  & 2.0       & 0.005 \\
          & v2 & 3        & 30  & 60  & 1.0       & 0.005 \\
          & v3 & 5        & 25  & 40  & 0.5       & 0.005 \\
          & v4 & 3        & 30  & 80  & 1.0       & 0.01  \\
          \midrule
\multirow{4}{*}{\rotatebox[origin=c]{90}{\textbf{\scriptsize WN18RR}}}     & v1 & 3        & 20  & 40  & 1.0       & 0.01  \\
          & v2 & 3        & 20  & 60  & 0.5       & 0.01  \\
          & v3 & 3        & 20  & 40  & 1.0       & 0.01  \\
          & v4 & 3        & 30  & 80  & 1.0       & 0.01  \\
          \midrule
\multirow{4}{*}{\rotatebox[origin=c]{90}{\textbf{\scriptsize NELL-995}}}   & v1 & 3        & 20  & 80  & 2.0       & 0.005 \\
          & v2 & 4        & 30  & 60  & 2.0       & 0.01  \\
          & v3 & 4        & 25  & 40  & 0.5       & 0.01  \\
          & v4 & 4        & 30  & 60  & 1.0       & 0.01 \\
          \bottomrule
\end{tabular}
\caption{\modelName's best-performing hyperparameters on FB15k-237 v1-4, WN18RR v1-4, and NELL-995 v1-4.}
\label{tab:hyperparas}
\end{table}

\begin{table*}[t]
\centering
\footnotesize
\setlength\tabcolsep{4pt}
\begin{tabular}{ccccccccccccc}
\toprule
              & \multicolumn{4}{c}{\textbf{FB15k-237}} & \multicolumn{4}{c}{\textbf{WN18RR}}    & \multicolumn{4}{c}{\textbf{NELL-995}}  \\
                                       \cmidrule(lr){2-5} \cmidrule(lr){6-9} \cmidrule(lr){10-13}
                & v1    & v2    & v3    & v4    & v1    & v2    & v3    & v4    & v1    & v2    & v3    & v4    \\
Seed 1                & 0.751 & 0.879 & 0.905 & 0.918 & 0.713 & 0.727 & 0.614 & 0.693 & 0.630 & 0.874 & 0.871 & 0.816 \\
Seed 2                & 0.744 & 0.892 & 0.908 & 0.916 & 0.707 & 0.726 & 0.574 & 0.690 & 0.650 & 0.860 & 0.893 & 0.808 \\
Seed 3                & 0.746 & 0.883 & 0.897 & 0.918 & 0.710 & 0.736 & 0.617 & 0.698 & 0.635 & 0.848 & 0.881 & 0.812 \\
\midrule
$\textit{mean}$ & 0.747 & 0.885 & 0.903 & 0.918 & 0.710 & 0.729 & 0.602 & 0.694 & 0.638 & 0.861 & 0.882 & 0.812 \\
$\textit{stdv}$ & 0.004 & 0.007 & 0.005 & 0.001 & 0.003 & 0.006 & 0.024 & 0.004 & 0.010 & 0.013 & 0.011 & 0.004 \\
\bottomrule
\end{tabular}
\caption{\modelName's benchmark Hits@10 scores on all seeds together with the mean ($\textit{mean}$) and standard deviation ($\textit{stdv}$) of Hits@10.}
\label{tab:BenchRes}
\end{table*}

Following \citet{DBLP:conf/icml/TeruDH20}, we manually tune \modelName's hyperparameters on the validation split of $\mathcal{G}_{\textit{Train}}$. We use the following ranges for the hyperparameters: the number of \modelName's layers $\textit{\#Layers} \in \{3, 4, 5\}$, the embedding dimensionality parameters $l \in \{20, 25, 30\}$ and $k \in \{40, 60, 80\}$, the loss margin $\lambda \in \{0.5, 1.0, 2.0\}$, and finally the learning rate $\textit{lr} \in \{0.005, 0.01\}$. We use the same batch and negative sampling size for all runs. In particular, we set the batch size to $1024$ and the negative sampling size to $100$. We report the best hyperparameters for \modelName\ split by each inductive benchmark in Table~\ref{tab:hyperparas}. Finally, we reuse the same hyperparameters for each of \modelName's ablations, namely, \modelName$_{\sf nL}$ and \modelName$^2$.

\subsection{Evaluation Protocol and Complete Results}
\label{app:Metrics}

Following the standard evaluation protocol for inductive KGC, introduced by \citet{DBLP:conf/icml/TeruDH20}, we evaluate \modelName's final performance on the test split of the testing graph by measuring the ranking quality of any test triple $r(e, f)$ over $50$ randomly sampled entities $e'_i \in \mathcal{E}$ and $f'_i \in \mathcal{E}$:  $r(e'_i, f)$ and $r(e, f'_i)$ for all $1 \leq i \leq 50$. Following \citet{DBLP:conf/icml/TeruDH20}, we report the Hits@10 metric, i.e., the proportion of true triples (those within the test split of the testing graph) among the predicted triples whose rank is at most $10$. 

\medskip
\noindent
Table \ref{tab:BenchRes} states \modelName's benchmark results over all inductive datasets, as well as their means and standard deviations.

\end{document}